\documentclass{article} %
\usepackage{iclr2027_conference,times}

\usepackage{amsmath,amsfonts,bm}

\def\eqref#1{equation~\ref{#1}}

\def\1{\bm{1}}

\DeclareMathAlphabet{\mathsfit}{\encodingdefault}{\sfdefault}{m}{sl}
\SetMathAlphabet{\mathsfit}{bold}{\encodingdefault}{\sfdefault}{bx}{n}

\newcommand{\R}{\mathbb{R}}

\usepackage{url}

\definecolor{mydarkred}{rgb}{0.6,0,0}
\definecolor{myblue}{HTML}{268BD2}

\usepackage{graphicx}
\usepackage[utf8]{inputenc} %
\usepackage[T1]{fontenc}    %
\usepackage{url}            %
\usepackage{booktabs}       %
\usepackage{amsfonts}       %
\usepackage{nicefrac}       %
\usepackage{microtype}      %
\usepackage{xcolor}         %
\usepackage{color}
\usepackage{multirow}
\usepackage{wrapfig}
\usepackage{mathtools}
\usepackage{subcaption}
\usepackage{amsmath}
\usepackage{amsthm}
\usepackage{algorithm}
\usepackage{colortbl}
\usepackage{adjustbox}
\usepackage{cuted}
 \usepackage{multicol}

\newtheorem{theorem}{Theorem}

\definecolor{textpurple}{HTML}{883BDD} 
\RequirePackage[dvipsnames]{xcolor}

\usepackage{enumitem}

\usepackage{pifont}

\usepackage[most]{tcolorbox}

\usepackage{tcolorbox}
\usepackage{listings}
\tcbuselibrary{breakable, skins}

\usepackage{algorithm}
\usepackage{algpseudocode}

\usepackage[colorlinks,
linkcolor=mydarkred,
urlcolor=RoyalBlue,
citecolor=myblue]{hyperref}
\usepackage{cleveref}
\usepackage{subcaption}
\usepackage{makecell}

\usepackage{listings,xcolor}
\definecolor{algcmt}{HTML}{287E7E}   %
\definecolor{algkw}{HTML}{1010EE}    %
\definecolor{algfn}{HTML}{D6336C}    %
\definecolor{algstr}{HTML}{E8860C}   %
\lstdefinestyle{pyalg}{
  language=Python,
  basicstyle=\fontsize{8.5pt}{10pt}\ttfamily,
  keywordstyle=\color{algkw},
  commentstyle=\color{algcmt},
  stringstyle=\color{algstr},
  emph={mean,sum,max,dot,norm,inv,relu,modality,damage,
        sum_h,sum_i,sum_a,mean_l},
  emphstyle=\color{algfn},
  alsoletter={_},
  literate={+}{{{\color{algkw}+}}}1
           {-}{{{\color{algkw}-}}}1
           {*}{{{\color{algkw}*}}}1
           {/}{{{\color{algkw}/}}}1,
  columns=fullflexible,
  keepspaces=true,
  showstringspaces=false,
  xleftmargin=.5em,
}

\newtcolorbox{promptbox}[1]{
    colback=cyan!0!white,
    colframe=cyan!60!black,
    fonttitle=\bfseries,
    title=#1,
    breakable,
    boxrule=0.8pt,
    arc=3pt,
    left=6pt, right=6pt, top=4pt, bottom=4pt
}

\lstdefinestyle{promptstyle}{
    basicstyle=\ttfamily\small,
    breaklines=true,
    backgroundcolor=\color{cyan!8!white},
    frame=none,
    xleftmargin=0pt,
    columns=flexible,
    keepspaces=true,
    literate={\\boxed}{\textbackslash boxed}6
             {<think>}{\textlangle think\textrangle}7
             {</think>}{\textlangle /think\textrangle}8
             {<answer>}{\textlangle answer\textrangle}8
             {</answer>}{\textlangle /answer\textrangle}9
}

\title{Tracing the Evidence: Faithful Token Attribution Through Vision-Language Reasoning}

\author{
Bowen Yuan$^{*}$, Danny Wang$^{*}$, Ruihong Qiu, Zijian Wang, Zi Huang\\
The University of Queensland\\
\texttt{\{bowen.yuan,danny.wang,r.qiu,zijian.wang,helen.huang\}@uq.edu.au}
}

\iclrfinalcopy %
\begin{document}

\maketitle
\lhead{Preprint.}
{\renewcommand{\thefootnote}{*}\footnotetext{Equal contribution.}}

\begin{abstract}
\vspace{-8pt}

Large vision-language models (LVLMs) exhibit strong reasoning capabilities, yet the visual and textual evidence supporting the generated responses remains difficult to identify.
Faithful token attribution explains an LVLM's response by assigning scores that rank image and prompt tokens by how much the model relies on them, such that removing higher-ranked tokens causes the likelihood of the generated response to drop more rapidly.
However, existing token-attribution methods have been developed mainly for text-based language models, and our empirical study reveals two challenges when complex multimodal sources are involved.
First, the joint image-text attribution can underrepresent visual evidence relative to text, obscuring the image regions supporting the response.
Second, visual evidence may influence the generated response through multiple intermediate reasoning paths, while existing methods trace only a limited subset of these paths, causing important visual contributions to be underestimated.
Motivated by these insights, we introduce \textsc{VTrace}, a multimodal token-attribution framework that traces input contributions through intermediate reasoning and calibrates attribution scores across modalities.
\textsc{VTrace} constructs pairwise attributions that highlight token-specific contributions and aggregates all forward attribution paths in closed form to account for both direct and indirect contributions.
Cross-modal calibration then rescales image and text attribution scores using modality contributions estimated from response-likelihood changes, enabling a unified ranking of input tokens.
Evaluations against seven baselines across six visual reasoning benchmarks demonstrate the superior attribution faithfulness.
Project page: \url{https://vtrace-attribution.github.io/}.
\end{abstract}

\vspace{-0.1cm}
\begin{center}
\vspace{-0.3cm}
\includegraphics[width=0.88\linewidth]{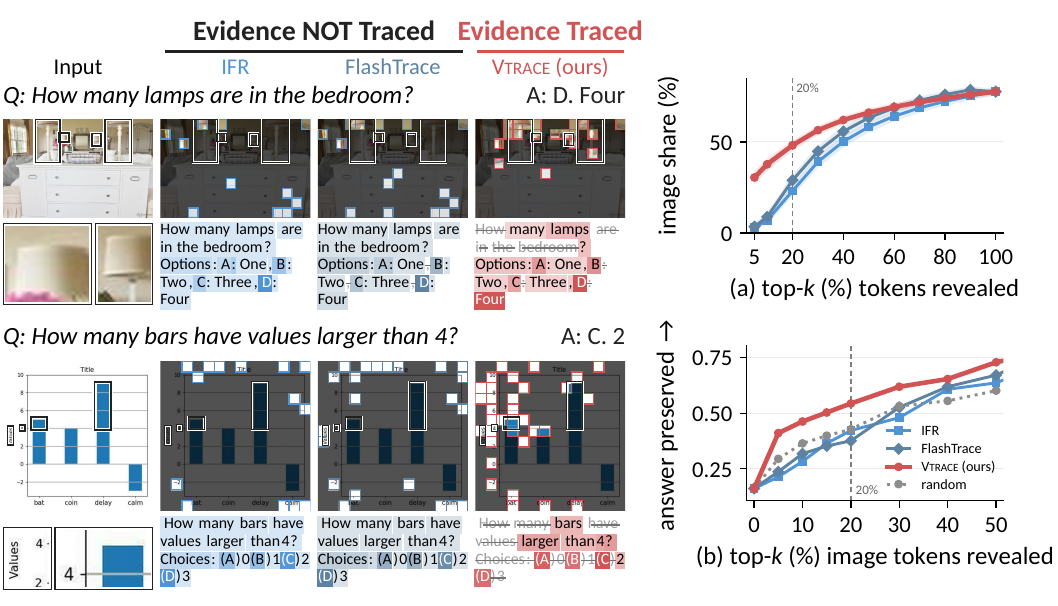}
\vspace{-1.0ex}
\captionof{figure}{
\textbf{Relevant visual evidence can be under-ranked in multimodal attribution.}
\textbf{Left:} Top-20\% attributed tokens in two examples, with color intensity indicating attribution strength. Black boxes mark the question-relevant region, enlarged below.
\textbf{Right:} (a) percentage of attributed image tokens and (b) normalized answer probability recovery across different selected token budgets.
}
\label{fig:Motivation_teaser}
\end{center}
\vspace{-1cm}

\section{Introduction} \label{sec:introduction}
\vspace{-0.2cm}
Large vision-language models (LVLMs) are increasingly capable of solving complex visual reasoning tasks by generating intermediate reasoning before arriving at a final answer~\citep{llava_cot, cot_vlm, multimodal_cot}.
However, observing the answer alone does not reveal which visual and textual evidence contributes to the answer, or how such evidence is used throughout the reasoning process~\citep{lvlm_interpret,glimpse}.
Explaining these predictions requires identifying how input evidence contributes to predictions throughout the reasoning~\citep{flashtrace, uppaal2026journey, wordsorvision}.
Token attribution provides a way to estimate these contributions: for a generated token or a target span, token attribution assigns preceding tokens an \emph{attribution score} reflecting the contribution to that output~\citep{attn_rollout, alti, attnLRP}.
A faithful attribution should accurately reflect how input evidence contributes to the selected tokens.

Existing attribution methods typically quantify token contributions using either model-internal signals or behavioral changes under intervention.
Internal-signal-based methods trace information propagation through transformer attention or token interactions to attribute predictions to preceding tokens~\citep{attn_rollout,alti,attnLRP}, whereas perturbation-based methods measure how modifying input tokens changes the model's output distribution~\citep{reagent}.
More recent methods extend attribution to the reasoning process, with FlashTrace~\citep{flashtrace} tracing influence through reasoning tokens and FlowTracer~\citep{flowtracer} modeling answer-directed information flow across the reasoning trace.

\begin{wrapfigure}{r}{0.5\textwidth}
\vspace{-0.5cm}
  \begin{center}
    \includegraphics[width=0.48\textwidth]{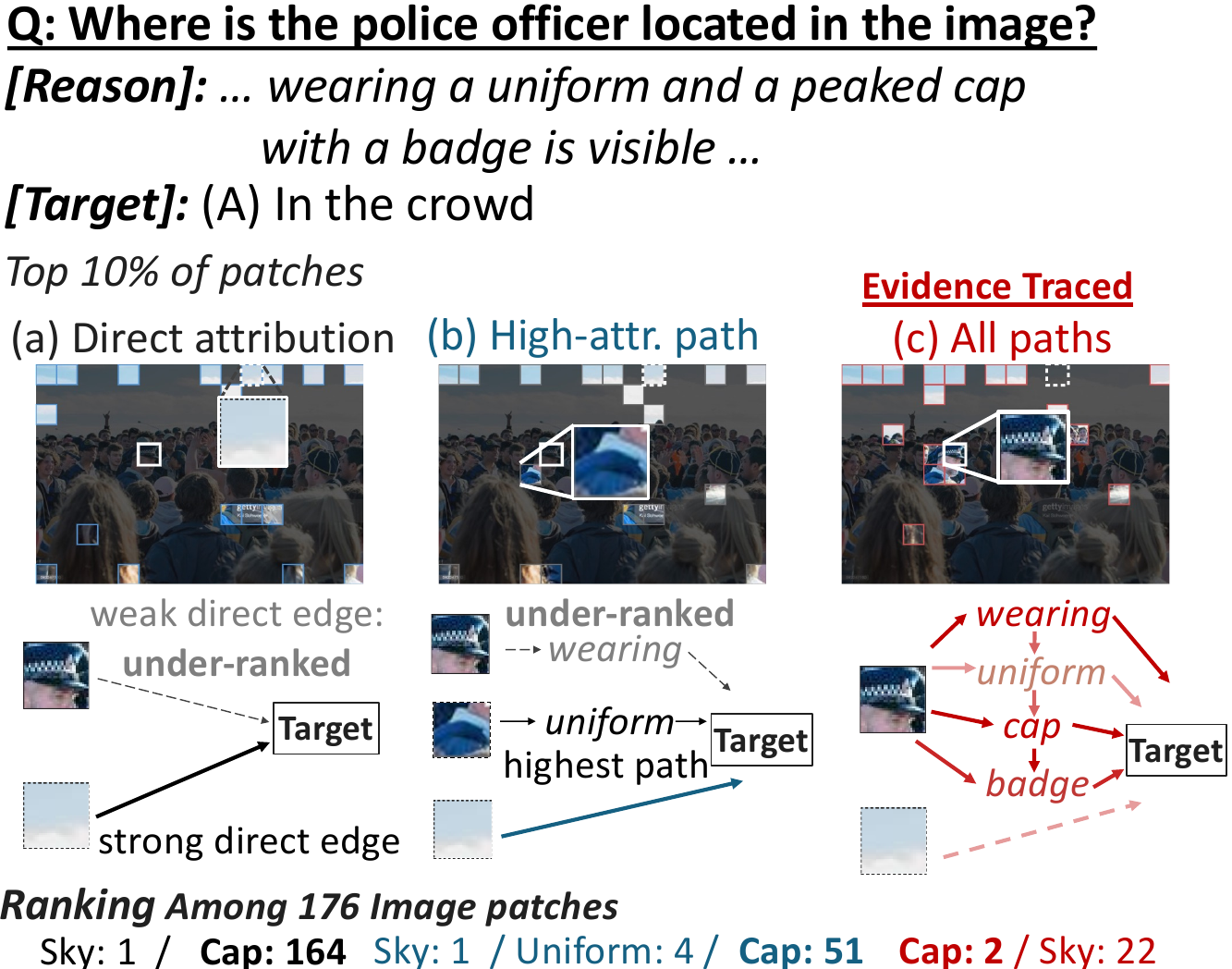}
  \end{center}
  \caption{Tracing only highly-attributed paths misses evidence that many weaker paths carry.
}
  \label{fig:challenge_1}
\vspace{-0.3cm}
\end{wrapfigure}
Despite the advances, attribution in LVLM reasoning presents two challenges.
\textbf{1) Visual evidence generally receives lower attribution scores when image and text tokens are scored together.}
Our empirical analysis reveals an imbalance in joint image-text attribution, where visual evidence are underrepresented relative to text.
Illustrated in~\Cref{fig:Motivation_teaser}, textual tokens such as answer options often receive higher attribution scores than image patches with question-relevant visual cues.
Consequently, relevant image patches can be ranked below less relevant text tokens, obscuring critical visual evidence that supports the model's reasoning.
\textbf{2) Visual evidence can be underestimated when only some of the reasoning paths to the target are traced.}
Relevant visual evidence may support the target through multiple intermediate reasoning tokens.
Shown in~\Cref{fig:challenge_1}, critical visual evidence from the officer's cap propagates through tokens such as \emph{wearing} and \emph{badge}. Direct attribution (a) ranks the cap only 164th among 176 image patches, while tracing the highest-attribution path (b) improves it to 51st. Aggregating all reasoning paths (c), however, raises it to 2nd, showing that limited path tracing substantially under-ranks evidence by capturing only part of its propagated contribution.

To address these challenges, we introduce \textsc{VTrace}, a multimodal token-attribution framework that traces visual and textual contributions through intermediate reasoning and calibrates attribution scores across modalities. To reduce the underrepresentation of visual evidence in joint image-text ranking, \textsc{VTrace} isolates token-specific contributions by removing shared components within each modality, and calibrates aggregated image and text attributions according to each modality's effect on response likelihood. To recover visual contributions distributed across intermediate reasoning, \textsc{VTrace} constructs pairwise token attributions over the full sequence and aggregates all direct and indirect reasoning paths in closed form, tracing the accumulated contributions back to the input tokens.
Together, our contributions are threefold:
\vspace{-0.2cm}
\begin{itemize}[leftmargin=*]
\item We empirically characterize two critical limitations when extending token attribution to LVLMs: incomplete recovery of visual contributions through intermediate reasoning and underrepresentation of visual evidence in joint image-text rankings.
\item We propose \textsc{VTrace}, which combines modality-aware pairwise attribution, closed-form path aggregation, and cross-modal calibration to trace source evidence through reasoning while making image and text attribution scores comparable.
\item We evaluate against seven baselines across six visual reasoning benchmarks. In joint image-text evaluation, \textsc{VTrace} improves mean RISE insertion AUC by 7.1\% and reduces deletion AUC by 18.0\%, demonstrating more faithful attribution with competitive computational efficiency.
The attribution-guided post-training experiment improves model's visual reasoning performance, suggesting the potential of \textsc{VTrace} as a learning signal for LVLM reasoning.
\end{itemize}

\vspace{-0.2cm}
\section{Related Work} \label{section: related work}
\vspace{-0.1cm}
Our work builds on three lines of research.
\textbf{1) Token attribution:} Existing methods estimate which source tokens influence a target token or output span using attention aggregation (Attention Rollout, ALTI)~\citep{attn_rollout,alti}, relevance propagation (AttnLRP)~\citep{attnLRP}, Hessian-based sensitivity (HETA)~\citep{HETA}, or perturbation-based attribution (ReAGent)~\citep{reagent}. Graph-based approaches instead trace contribution paths, where IFR builds a graph over token states and components~\citep{IFR}, FlashTrace recursively tracks attribution through generated tokens~\citep{flashtrace}, and FlowTracer models conserved flow over an attention graph~\citep{flowtracer}.
\textbf{2) Multimodal attribution:} Chefer et al.~\citep{Chefer_2021_ICCV} propagate  relevance through self- and cross-attention to attribute multimodal predictions to their inputs. LVLM-Interpret~\citep{lvlm_interpret} extends attention- and relevance-based analysis to autoregressive LVLMs, while GLIMPSE~\citep{glimpse} aggregates gradient-weighted attention across layers and generated tokens for response-level attribution. \textbf{3) Attribution through LVLM reasoning:} Multimodal-CoT and LLaVA-CoT generate intermediate reasoning from image and text before answering~\citep{multimodal_cot,llava_cot}, motivating studies of whether such reasoning faithfully reflects the evidence used by the model~\citep{cot_vlm,vlm_cot_faithfulness}. Existing attribution methods, however, only partially capture how multimodal evidence propagates through intermediate reasoning and do not address attribution-scale differences between image and text tokens. \textsc{VTrace} addresses both by aggregating direct and indirect influence across the reasoning trace and calibrating cross-modal attribution scores. Detailed review is in Appendix~\ref{appendix:extended_related_work}.

\section{Critical Challenges in Multimodal Attribution}
\label{sec:challenges}
\subsection{Preliminaries}
\textbf{Autoregressive LVLM Generation.} Given an image $I$ and a text prompt $q$, an LVLM with parameters $\theta$ generates a response $y=(y_1,\ldots,y_n)$ containing intermediate reasoning and a final answer:
\vspace{-0.1cm}
\begin{equation}
p_\theta(y\mid I,q)=\prod_{t=1}^{n}p_\theta(y_t\mid I,q,y_{<t}).
\label{eq:lvlm_generation}
\end{equation}
Under this autoregressive process, input evidence can contribute to later predictions both directly and through intermediate reasoning tokens.
Our goal is to attribute a selected output token or span, ranging from an individual answer token to the complete response, to its supporting input evidence by accounting for these direct and indirect contributions.

\begin{wrapfigure}{r}{0.50\textwidth}
    \vspace{-0.4cm}
    \centering
    \includegraphics[width=0.48\textwidth]{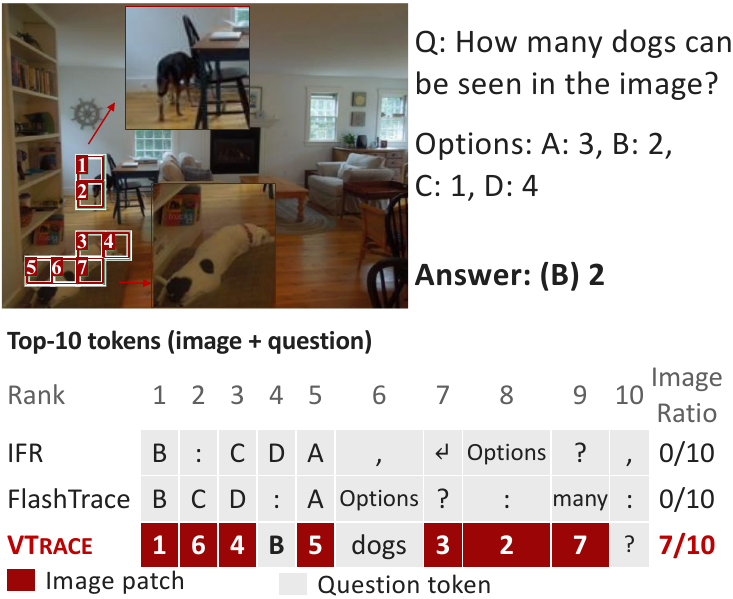}
    \caption{Existing methods overlook visual cues.}
    \label{fig:challenge_2}
    \vspace{-0.2cm}
\end{wrapfigure}
\textbf{Pairwise Token Attribution.}
Let $X=(x_1,\ldots,x_T)$ denote the complete sequence of input and generated tokens, with $\mathcal I$ and $\mathcal T$ indexing image and prompt-text tokens, respectively.
We estimate direct token-to-token contributions using a pairwise attribution matrix $W\in\mathbb{R}_{\geq 0}^{T\times T}$, where $W_{ij}$ scores the contribution from an earlier source token $x_i$ to a later receiver token $x_j$.
To trace contributions forward through the sequence, we retain only entries with $i<j$ and set $W_{ij}=0$ otherwise.
Thus, $W$ is strictly upper triangular.
A generated token can receive contributions from earlier tokens
and subsequently contribute to later tokens, allowing input evidence to propagate through intermediate reasoning.
\Cref{sec:direct_attribution} defines how the entries of $W$ are computed.

\subsection{Visual Evidence Is Obscured by Textual Attribution}
\label{sec:modality_challenge}
\vspace{-0.1cm}
In multimodal attribution, image and text tokens are jointly ranked according to their contributions, yet existing methods exhibit a strong attribution imbalanced toward text tokens. As shown in~\Cref{fig:Motivation_teaser} Right, existing methods select only a small fraction of image tokens among their highest-ranked tokens. This imbalance is also evident in~\Cref{fig:challenge_2}: when the model correctly counts two dogs, neither IFR nor FlashTrace includes an image patch in its top-10 tokens, while option letters, punctuation, and question tokens occupy much of the ranking.
These observations reveal attribution imbalance that pushes important visual evidence below textual tokens in attribution, causing the resulting attribution to underrepresent the visual information in multimodal reasoning.
In contrast, \textsc{VTrace} aligns image and text scores by their effect on response likelihood, bringing seven image patches into the top 10, producing a ranking that better reflects the visual evidence required for the answer (\Cref{sec:modality_alignment}).
\vspace{-0.2cm}
\subsection{Limited Reasoning-Path Tracing Marginalizes Visual Evidence} \label{subsec:limited_hops_challenge}
Existing attribution methods trace only direct or limited-hop contributions, making it difficult to trace distributed evidence back to its original sources.
For instance, IFR measures direct input to answer contributions~\citep{IFR}, while FlashTrace recursively propagates attribution through a limited number of reweighted reasoning steps~\citep{flashtrace}.
However, evidence from input tokens can be progressively integrated into subsequent generated tokens during reasoning, which in turn contribute to later predictions.
As a result, the final prediction receives strong attribution from intermediate reasoning tokens while assigning weak attribution to the upstream visual evidence from which they originated. This is depicted in the aforementioned~\Cref{fig:challenge_1}.
\textsc{VTrace} instead aggregates evidence over all direct and indirect paths (\Cref{sec:multihop}), allowing information propagated through the reasoning trace to be traced back to the crucial image and text tokens.

\vspace{-0.3cm}
\section{Methodology} \label{section: method}
\vspace{-0.2cm}
To address these issues, \textsc{VTrace} proceeds in three stages.
It first constructs a modality-aware pairwise attribution matrix, then aggregates direct and indirect contributions through intermediate reasoning, and finally calibrates image and text attribution scores for joint ranking.
Figure~\ref{fig:framework} presents the framework, and Algorithm is in Appendix Algorithm~\ref{alg:VTrace}.
\vspace{-0.2cm}
\subsection{Modality-Aware Pairwise Token Attribution}
\label{sec:direct_attribution}

\textbf{Tracing token-specific information.}
Consider layer $\ell$ of an LVLM with $H$ attention heads and hidden dimension $d$.
For a receiver token $x_j$, the attention block produces the update $\Delta_j^{(\ell)}$:\
\vspace{-0.1cm}
\begin{equation}
\Delta_j^{(\ell)} = \sum_{h=1}^{H}\sum_{i\leq j} \alpha_{ji}^{(\ell,h)} f_i^{(\ell,h)},
\label{eq}
\vspace{-0.1cm}
\end{equation}
where $\alpha_{ji}^{(\ell,h)}$ is the attention weight assigned by receiver token $x_j$ to source token $x_i$, and $f_i^{(\ell,h)} = v_i^{(\ell,h)}O^{(\ell,h)} \in \mathbb{R}^{d}$ is the output-projected value vector of $x_i$ at head $h$.

As shown in~\Cref{sec:modality_challenge}, directly transferring token-attribution signals across modalities leads to biased multimodal attribution.
We instead characterize a source by the information it contributes over the common component of its own modality.
For a source $x_i$ in either modality, we define the modality-aware write $\widetilde f_{i}^{(\ell,h)}$~\footnote{We use \emph{write} in the residual-stream sense: a component
reads from the residual stream and contributes an update back to it~\citep{AnthropicTransformerCircuits}.} by subtracting its modality's mean output-projected value:
\begin{equation} 
\widetilde f_{i}^{(\ell,h)} = f_i^{(\ell,h)} - \frac{1}{|\mathcal M(i)|} \sum\nolimits_{r\in\mathcal M(i)} f_r^{(\ell,h)},
\label{eq:anchor} 
\end{equation}
where $\mathcal M(i)=\mathcal I$ for $i\in\mathcal I$ and $\mathcal M(i)=\mathcal T$ for $i\in\mathcal T$, denoting the set of input tokens belonging to the same modality as token $x_i$.
This removes the modality-wise mean component from each source write, emphasizing the component that distinguishes $x_i$ from other tokens within the same modality.

To quantify the attribution from source $x_i$ to receiver $x_j$, we measure how strongly the source-specific write aligns with the update at $x_j$.
For $i<j$, the corresponding entry of $W$ is:

\begin{equation} W_{ij} = \frac{1}{L} \sum_{\ell=1}^{L} \frac{ \left[ \sum_{h=1}^{H} \alpha_{ji}^{(\ell,h)} \left\langle \widetilde f_{i}^{(\ell,h)}, \Delta_j^{(\ell)} \right\rangle \right]_{+} }{ \left\|\Delta_j^{(\ell)}\right\|_2 }, \qquad i<j, \label{eq:direct_edge} \end{equation}
and $W_{ij}=0$ otherwise. The inner product $\left\langle\cdot, \cdot\right\rangle$ measures alignment between the source vector and the receiver-token update, while $[\cdot]_+$ retains positive contributions. Thus, $W\in\mathbb{R}_{\geq0}^{T\times T}$ is a strictly upper-triangular matrix containing the direct pairwise token attribution weights. 

\begin{figure}[t!]
    \centering
    \includegraphics[width=1\linewidth]{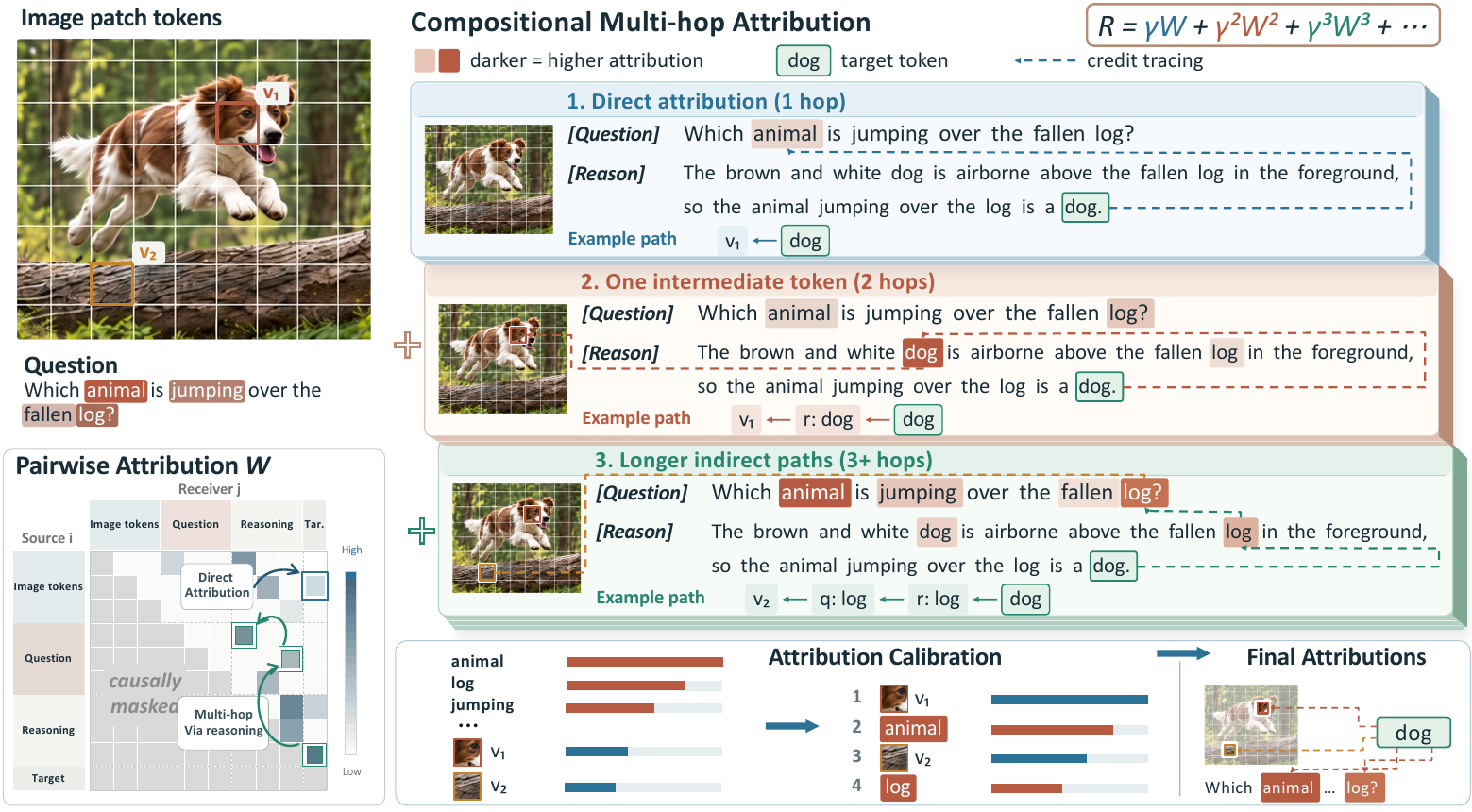}
    \vspace{-0.1cm}
    \caption{
    Overview of \textsc{VTrace}. \textsc{VTrace} (1) measures each pair of tokens against the average contribution of its own modality and (2) aggregates all direct and indirect paths in closed form, (3) enabling image patches and words to be faithfully ranked together. 
    }
    \label{fig:framework}
     \vspace{-0.7cm}
\end{figure}

\subsection{Compositional Multi-Hop Attribution: Direct and Indirect Contributions}
\label{sec:multihop}
The pairwise attribution matrix $W$ quantifies direct contributions between source and receiver tokens. A source token $x_i$, however, can contribute to a later target $x_j$ either directly or through intermediate tokens, including generated reasoning tokens. We refer to aggregating all such direct and indirect paths as \emph{compositional multi-hop} (CMH) attribution. \textsc{VTrace} computes CMH attribution over $W$ using a Katz-based formulation~\citep{katz_index}:
\begin{equation}
R_{ij} = \gamma\widehat W_{ij} + \gamma^2(\widehat W^2)_{ij} + \gamma^3(\widehat W^3)_{ij} + \cdots = \sum\nolimits_{\tau=1}^{T-1}\gamma^{\tau} (\widehat W^\tau)_{ij},
\label{eq:multihop_pair}
\end{equation}
where $\widehat W = \frac{W}{\max_{i,j}W_{ij}+\epsilon}$ is the normalized direct-attribution matrix, and $\gamma\geq0$ weights paths by length. Here, $\widehat W_{ij}$ captures direct attribution, while $(\widehat W^\tau)_{ij}$ aggregates attribution over all length-$\tau$ paths from $x_i$ to $x_j$. The finite path expansion admits the exact closed form:
\begin{equation}
R
= \sum\nolimits_{\tau=1}^{T-1}\gamma^{\tau}\widehat W^\tau
= (\mathbf I-\gamma\widehat W)^{-1}-\mathbf I
\label{eq:katz_closed_form}
\end{equation}
The proof is in Appendix~\ref{proof:W_closed_form}. Each $R_{ij}$ aggregates the direct and indirect attribution from source $x_i$ to receiver $x_j$ represented by
$\widehat W_{ij}$.

\textbf{Scoring tokens by incoming and outgoing attribution.}
The matrix $R$ assigns attribution to pairs of tokens. To obtain one score for each token, we consider every token $x_r$ that precedes the selected receiver positions $x_j \in \mathcal A$. Its incoming attribution is the total path attribution reaching $x_r$ from earlier tokens, and its outgoing attribution is the total path attribution from $x_r$ to the selected receivers:
\begin{equation}
\operatorname{In}(x_r) = \sum\nolimits_{i<r}R_{ir},
\qquad
\operatorname{Out}_{\mathcal A}(x_r) = \sum\nolimits_{j\in\mathcal A}R_{rj}.
\label{eq:incoming_outgoing_flow}
\end{equation}
We define the \textbf{uncalibrated attribution score} of $x_r$ as:
$
u_r = \bigl(1+\operatorname{In}(x_r)\bigr)\operatorname{Out}_{\mathcal A}(x_r).
\label{eq:flow_importance}
$

This score includes paths that begin at $x_r$ and paths that pass through $x_r$ before reaching a selected receiver $x_j$. The added $1$ preserves the contribution of a token with no incoming attribution.

\subsection{Aligning Image and Text Attribution Scores}
\label{sec:modality_alignment}

The multi-hop stage produces an uncalibrated attribution score $u_r$ for each upstream token. To jointly attribute image and textual tokens, we calibrate their total attribution according to each modality's contribution to the generated output.
For a modality set $\mathcal S$, we measure this contribution by the decrease in teacher-forced log-probability when its tokens are masked:
\begin{equation}
D_{\mathcal S}
= \log p_{\theta}(y\mid X)
- \log p_{\theta}\bigl(y\mid\operatorname{pad}(\mathcal S)\bigr),
\qquad
\mathcal S\in\{\mathcal I,\mathcal T,\mathcal I\cup\mathcal T\}.
\label{eq:modality_damage}
\end{equation}
Here, $\operatorname{pad}(\mathcal S)$ replaces positions in $\mathcal S$ with the model's pad-token embedding, so $D_{\mathcal S}$ measures the contribution of $\mathcal S$ to the output. To account for image-text interactions, we use exact two-player Shapley values to divide their joint contribution:
\begin{equation}
\phi_{\mathcal I} = \frac{1}{2}\left(D_{\mathcal I}+D_{\mathcal I\cup\mathcal T}-D_{\mathcal T}\right),
\qquad
\phi_{\mathcal T} = \frac{1}{2}\left(D_{\mathcal T}+D_{\mathcal I\cup\mathcal T}-D_{\mathcal I}\right).
\label{eq:modality_shapley}
\end{equation}
Let $s=\phi_{\mathcal I}/(\phi_{\mathcal I}+\phi_{\mathcal T})$ be the image share. We then rescale image-token scores by:
\begin{equation}
\lambda = \frac{s}{1-s}\frac{\sum_{j\in\mathcal T}u_j}{\sum_{i\in\mathcal I}u_i},
\label{eq:modality_rescale}
\end{equation}
and leave user-text scores unchanged. 
This matches the total image-text attribution ratio to their estimated modality contributions while preserving the ranking within each modality.

\textbf{Final \textsc{VTrace} attribution score.} For each input token $r$, the final \textsc{VTrace} attribution score is:
\begin{equation}
u_{\textsc{VTrace}}
\begin{cases}
\lambda u_r, & r\in\mathcal I,\\
u_r, & r\in\mathcal T,
\end{cases}
\qquad r\in\mathcal I\cup\mathcal T.
\label{eq:final_VTrace_score}
\end{equation}
Therefore, \textsc{VTrace} uses $R_{ij}$ to aggregate direct and indirect attribution from source $x_i$ to target $x_j$, and $u_{\textsc{VTrace}}$ as the final cross-modal attribution score to jointly rank image and prompt-text tokens.

\section{Experiments} \label{sec:experiments}

\textbf{Benchmark Datasets and Baseline Methods (Appendix~\ref{appendix:baseline_and_benchmarks}).}
We conduct experiments on 6 visual reasoning benchmarks to comphrehensively evaluate our method's capability: MMStar~\citep{mmstar}, MathVista~\citep{mathvista}, MMMU~\citep{MMMU}, MMMU-pro~\citep{mmmu-pro}, MathVerse~\citep{mathverse} and VisualPuzzles~\citep{visualPuzzles}. They cover visually dependent understanding, mathematical reasoning, and knowledge-based visual reasoning, providing diverse settings for evaluating attribution across multimodal reasoning processes.
We compare our method against existing attribution methods, including ReAGent~\citep{reagent}, HETA~\citep{HETA}, FlowTracer~\citep{flowtracer}, IFR~\citep{IFR}, Attn Rollout~\citep{attn_rollout}, AttnLRP~\citep{attnLRP}, FlashTrace~\citep{flashtrace}.

\textbf{Implementation Details.}
\label{sec:implementation_details}
Experiments are conducted on multiple scales of Qwen3-VL~\citep{qwen3vl} and InternVL3.5~\citep{internvl_3.5} (default Qwen3-VL-8B).
For each sample, we generate one response containing
intermediate reasoning and a final answer, and keep
the response fixed across all attribution methods to ensure fairness.
The main evaluation attributes the complete response to the input tokens.
Following prior work on attribution for reasoning models~\citep{flashtrace}, we evaluate attribution faithfulness using RISE~\citep{rise} and MAS~\citep{MAS}
insertion and deletion metrics, where tokens are ranked by attribution score and progressively removed or restored. 
Deletion measures how quickly the response degrades when highly attributed tokens are removed, while insertion measures how quickly it recovers when they are restored.
We report both metrics under \textbf{image-only} and \textbf{joint} settings, ranking visual tokens alone or visual and textual tokens together, respectively. For an input perturbed at step $k$, we score the model using the normalized likelihood of the original generated trace $y$:
$
f(X)_k = \exp\!\left( \frac{1}{n_{\mathrm{gen}}} \sum_{t=1}^{n_{\mathrm{gen}}} \log p_\theta \bigl(y_t \mid X^{(k)},y_{<t}\bigr) \right),
$
where $X^{(k)}$ denotes the perturbed context at step $k$.
Detailed implementation and evaluation setups are provided in Appendices~\ref{appendix:experimental_details} and~\ref{appendix:evaluaton_details}.

\begin{table}[t]
\centering
\caption{
\textbf{RISE attribution faithfulness} on Qwen3-VL-8B across six benchmarks under
\textbf{Image} and \textbf{Joint} perturbation settings.
We report RISE insertion (\textcolor{teal}{Ins.$\uparrow$}, higher is better)
and deletion (\textcolor{BurntOrange}{Del.$\downarrow$}, lower is better) AUC. The Image variant perturbs image patch tokens only, while Joint perturbs both image and text tokens.
\textcolor{Maroon}{\textbf{Best}} and \textcolor{NavyBlue}{\underline{Runner-up}} are highlighted.
}
\label{tab:rise_faithfulness_image_joint}
\resizebox{0.9\linewidth}{!}{
\renewcommand{\arraystretch}{1.08}
\begin{tabular}{lll|cccccccc}
\toprule
Dataset & Setting & \textbf{RISE}
& \textbf{ReAGent}
& \textbf{HETA}
& \textbf{FlowTracer}
& \textbf{IFR}
& \textbf{Attn Rollout}
& \textbf{AttnLRP}
& \textbf{FlashTrace}
& \cellcolor{blue!10}\textbf{\textsc{VTrace}} \\
\midrule

\multirow{4}{*}{MMStar}
& \multirow{2}{*}{Image}
& \textcolor{teal}{Ins.$\uparrow$}
& 0.505 & 0.497 & 0.532 & 0.539 & 0.539
& \underline{\textcolor{NavyBlue}{0.563}} & 0.555
& \cellcolor{blue!10}\textcolor{Maroon}{\textbf{0.600}} \\

&
& \textcolor{BurntOrange}{Del.$\downarrow$}
& 0.458 & 0.463 & 0.412 & 0.409 & 0.439
& \underline{\textcolor{NavyBlue}{0.384}} & 0.394
& \cellcolor{blue!10}\textcolor{Maroon}{\textbf{0.352}} \\

\cmidrule(lr){2-11}

& \multirow{2}{*}{Joint}
& \textcolor{teal}{Ins.$\uparrow$}
& 0.371 & 0.489 & 0.506 & 0.507 & 0.501
& 0.529 & \underline{\textcolor{NavyBlue}{0.554}}
& \cellcolor{blue!10}\textcolor{Maroon}{\textbf{0.581}} \\

&
& \textcolor{BurntOrange}{Del.$\downarrow$}
& 0.325 & 0.245 & 0.226 & 0.222 & 0.273
& 0.219 & \underline{\textcolor{NavyBlue}{0.204}}
& \cellcolor{blue!10}\textcolor{Maroon}{\textbf{0.180}} \\
\cmidrule(lr){1-11}

\multirow{4}{*}{MathVista}
& \multirow{2}{*}{Image}
& \textcolor{teal}{Ins.$\uparrow$}
& 0.500 & 0.514 & 0.579 & 0.591 & 0.584
& 0.599 & \underline{\textcolor{NavyBlue}{0.600}}
& \cellcolor{blue!10}\textcolor{Maroon}{\textbf{0.662}} \\

&
& \textcolor{BurntOrange}{Del.$\downarrow$}
& 0.447 & 0.444 & 0.368 & 0.363 & 0.394
& 0.358 & \underline{\textcolor{NavyBlue}{0.354}}
& \cellcolor{blue!10}\textcolor{Maroon}{\textbf{0.311}} \\

\cmidrule(lr){2-11}

& \multirow{2}{*}{Joint}
& \textcolor{teal}{Ins.$\uparrow$}
& 0.382 & 0.467 & 0.517 & 0.516 & 0.501
& 0.512 & \underline{\textcolor{NavyBlue}{0.574}}
& \cellcolor{blue!10}\textcolor{Maroon}{\textbf{0.647}} \\

&
& \textcolor{BurntOrange}{Del.$\downarrow$}
& 0.345 & 0.307 & 0.274 & 0.269 & 0.331
& 0.274 & \underline{\textcolor{NavyBlue}{0.241}}
& \cellcolor{blue!10}\textcolor{Maroon}{\textbf{0.178}} \\
\cmidrule(lr){1-11}

\multirow{4}{*}{MMMU}
& \multirow{2}{*}{Image}
& \textcolor{teal}{Ins.$\uparrow$}
& 0.543 & 0.564 & 0.595 & 0.610
& \underline{\textcolor{NavyBlue}{0.627}}
& 0.612 & 0.618
& \cellcolor{blue!10}\textcolor{Maroon}{\textbf{0.665}} \\

&
& \textcolor{BurntOrange}{Del.$\downarrow$}
& 0.482 & 0.480 & 0.438 & 0.425 & 0.444
& 0.419 & \underline{\textcolor{NavyBlue}{0.417}}
& \cellcolor{blue!10}\textcolor{Maroon}{\textbf{0.379}} \\

\cmidrule(lr){2-11}

& \multirow{2}{*}{Joint}
& \textcolor{teal}{Ins.$\uparrow$}
& 0.368 & 0.589 & 0.583 & 0.608 & 0.604
& 0.615 & \underline{\textcolor{NavyBlue}{0.629}}
& \cellcolor{blue!10}\textcolor{Maroon}{\textbf{0.660}} \\

&
& \textcolor{BurntOrange}{Del.$\downarrow$}
& 0.315 & 0.216 & 0.204 & 0.195 & 0.241
& 0.191 & \underline{\textcolor{NavyBlue}{0.178}}
& \cellcolor{blue!10}\textcolor{Maroon}{\textbf{0.159}} \\
\cmidrule(lr){1-11}

\multirow{4}{*}{MMMU-Pro}
& \multirow{2}{*}{Image}
& \textcolor{teal}{Ins.$\uparrow$}
& 0.527 & 0.542 & 0.570 & 0.572
& \underline{\textcolor{NavyBlue}{0.604}}
& 0.593 & 0.583
& \cellcolor{blue!10}\textcolor{Maroon}{\textbf{0.645}} \\

&
& \textcolor{BurntOrange}{Del.$\downarrow$}
& 0.472 & 0.472 & 0.442 & 0.435 & 0.444
& \underline{\textcolor{NavyBlue}{0.413}} & 0.429
& \cellcolor{blue!10}\textcolor{Maroon}{\textbf{0.378}} \\

\cmidrule(lr){2-11}

& \multirow{2}{*}{Joint}
& \textcolor{teal}{Ins.$\uparrow$}
& 0.371 & 0.558 & 0.532 & 0.539 & 0.554
& 0.559 & \underline{\textcolor{NavyBlue}{0.584}}
& \cellcolor{blue!10}\textcolor{Maroon}{\textbf{0.619}} \\

&
& \textcolor{BurntOrange}{Del.$\downarrow$}
& 0.322 & 0.241 & 0.242 & 0.242 & 0.269
& 0.229 & \underline{\textcolor{NavyBlue}{0.212}}
& \cellcolor{blue!10}\textcolor{Maroon}{\textbf{0.181}} \\
\cmidrule(lr){1-11}

\multirow{4}{*}{MathVerse}
& \multirow{2}{*}{Image}
& \textcolor{teal}{Ins.$\uparrow$}
& 0.523 & 0.565 & 0.636 & 0.640 & 0.624
& 0.610 & \underline{\textcolor{NavyBlue}{0.644}}
& \cellcolor{blue!10}\textcolor{Maroon}{\textbf{0.684}} \\

&
& \textcolor{BurntOrange}{Del.$\downarrow$}
& 0.454 & 0.430
& \underline{\textcolor{NavyBlue}{0.341}}
& 0.345 & 0.384 & 0.360
& \underline{\textcolor{NavyBlue}{0.341}}
& \cellcolor{blue!10}\textcolor{Maroon}{\textbf{0.296}} \\

\cmidrule(lr){2-11}

& \multirow{2}{*}{Joint}
& \textcolor{teal}{Ins.$\uparrow$}
& 0.383 & 0.495 & 0.539 & 0.540 & 0.513
& 0.547 & \underline{\textcolor{NavyBlue}{0.568}}
& \cellcolor{blue!10}\textcolor{Maroon}{\textbf{0.606}} \\

&
& \textcolor{BurntOrange}{Del.$\downarrow$}
& 0.330 & 0.314 & 0.279 & 0.283 & 0.342
& 0.274 & \underline{\textcolor{NavyBlue}{0.241}}
& \cellcolor{blue!10}\textcolor{Maroon}{\textbf{0.187}} \\
\cmidrule(lr){1-11}

\multirow{4}{*}{VisualPuzzles}
& \multirow{2}{*}{Image}
& \textcolor{teal}{Ins.$\uparrow$}
& 0.456 & 0.461 & 0.468 & 0.511 & 0.525
& \underline{\textcolor{NavyBlue}{0.531}} & 0.510
& \cellcolor{blue!10}\textcolor{Maroon}{\textbf{0.557}} \\

&
& \textcolor{BurntOrange}{Del.$\downarrow$}
& 0.409 & 0.405 & 0.379 & 0.354 & 0.380
& \underline{\textcolor{NavyBlue}{0.324}} & 0.358
& \cellcolor{blue!10}\textcolor{Maroon}{\textbf{0.309}} \\

\cmidrule(lr){2-11}

& \multirow{2}{*}{Joint}
& \textcolor{teal}{Ins.$\uparrow$}
& 0.376 & 0.488 & 0.477 & 0.504 & 0.523
& 0.529 & \underline{\textcolor{NavyBlue}{0.530}}
& \cellcolor{blue!10}\textcolor{Maroon}{\textbf{0.571}} \\

&
& \textcolor{BurntOrange}{Del.$\downarrow$}
& 0.348 & 0.288 & 0.291 & 0.283 & 0.283
& \underline{\textcolor{NavyBlue}{0.255}} & 0.268
& \cellcolor{blue!10}\textcolor{Maroon}{\textbf{0.207}} \\

\bottomrule
\end{tabular}
}
\vspace{-0.5cm}
\end{table}
\subsection{Quantitative Analysis}

\begin{wraptable}{r}{0.5\textwidth}
  \centering
  \vspace{-0.35cm}
  \caption{Ablation study on different components}
  \label{tab:component_ablation}
  \small
  \setlength{\tabcolsep}{4pt}
  \renewcommand{\arraystretch}{1.1}
  \resizebox{\linewidth}{!}{%
  \begin{tabular}{lcc|cc}
  \toprule
  \multirow{2}{*}{Method}
  & \multicolumn{2}{c|}{Image}
  & \multicolumn{2}{c}{Joint} \\
  \cmidrule(lr){2-3}\cmidrule(lr){4-5}
  & Ins.\,$\uparrow$ & Del.\,$\downarrow$ & Ins.\,$\uparrow$ & Del.\,$\downarrow$ \\
  \midrule
  \textbf{\textsc{VTrace}}                       & \textbf{0.631} & \textbf{0.327} & \textbf{0.610} & \textbf{0.187} \\
  \textbf{w/o centering}          & 0.625 & 0.340 & 0.599 & 0.200 \\
  \textbf{w/o multi-hop}       & 0.578 & 0.362 & 0.470 & 0.273 \\
  \textbf{w/o calibration}     & 0.631 & 0.327 & 0.602 & 0.200 \\
  \bottomrule
  \end{tabular}
  }
  \vspace{-0.3cm}
\end{wraptable}

\textbf{\textsc{VTrace} consistently outperforms existing attribution methods under both Image and Joint evaluation.}
As shown in~\Cref{tab:rise_faithfulness_image_joint}, \textsc{VTrace} outperforms all baselines across six benchmarks. Under the Image setting, \textsc{VTrace} improves average RISE insertion by 6.8\% and reduces deletion by 9.3\%. Under the Joint setting, it improves insertion by 7.1\% and reduces deletion by 18.0\%.
These consistent gains indicate that \textsc{VTrace} more faithfully identifies and ranks the visual and textual tokens contributing to the model response. Moreover, improvements in both insertion and deletion further show that highly ranked tokens are more effective at recovering the response when restored and disrupting it when removed.
For clarity, the main table reports RISE only, with complete RISE and MAS results provided in Appendix~\ref{appendix:main_results}.

\textbf{Ablation study.} We conduct ablation study on each component of \textsc{VTRACE}, including modality centering, compositional multi-hop aggregation and cross-modal calibration.
As reported in~\Cref{tab:component_ablation}, removing each component lowers attribution faithfulness.
Specifically, modality centering and multi-hop aggregation improve faithfulness in both variants. Calibration further improves Joint attribution while preserving the relative ranking of image tokens, complementing the token level attribution with cross-modal calibration. 
The sensitivity study and  is in Appendix~\ref{appendix:hyperparameter_ifr}.

\textbf{Efficiency Analysis.}
\textsc{VTrace} efficiently scales to full-trace attribution because it constructs the attribution matrix once and then reuses it to attribute every generated token.
For all methods, we measure the full attribution runtime, from processing the fixed response to producing the final attribution scores.
As a result, moving from one target span to all generated tokens adds only minimal overhead, increasing medium runtime from 0.24 s to 0.26 s (\Cref{fig:efficiency}(a)). Notably, this efficiency is achieved without sacrificing faithfulness, where~\Cref{fig:efficiency}(b) shows that \textsc{VTrace} attains the best faithfulness at substantially lower attribution time than baselines. Moreover, \textbf{\textsc{VTrace} scales favorably in memory and runtime}, using 31 GB at around 3,000 tokens versus 87 GB for AttnLRP (\Cref{fig:efficiency}(c)) and remaining faster than baselines as sequence length grows (\Cref{fig:efficiency}(d)).

\begin{figure*}[t]
    \centering
    \includegraphics[width=0.95\linewidth]{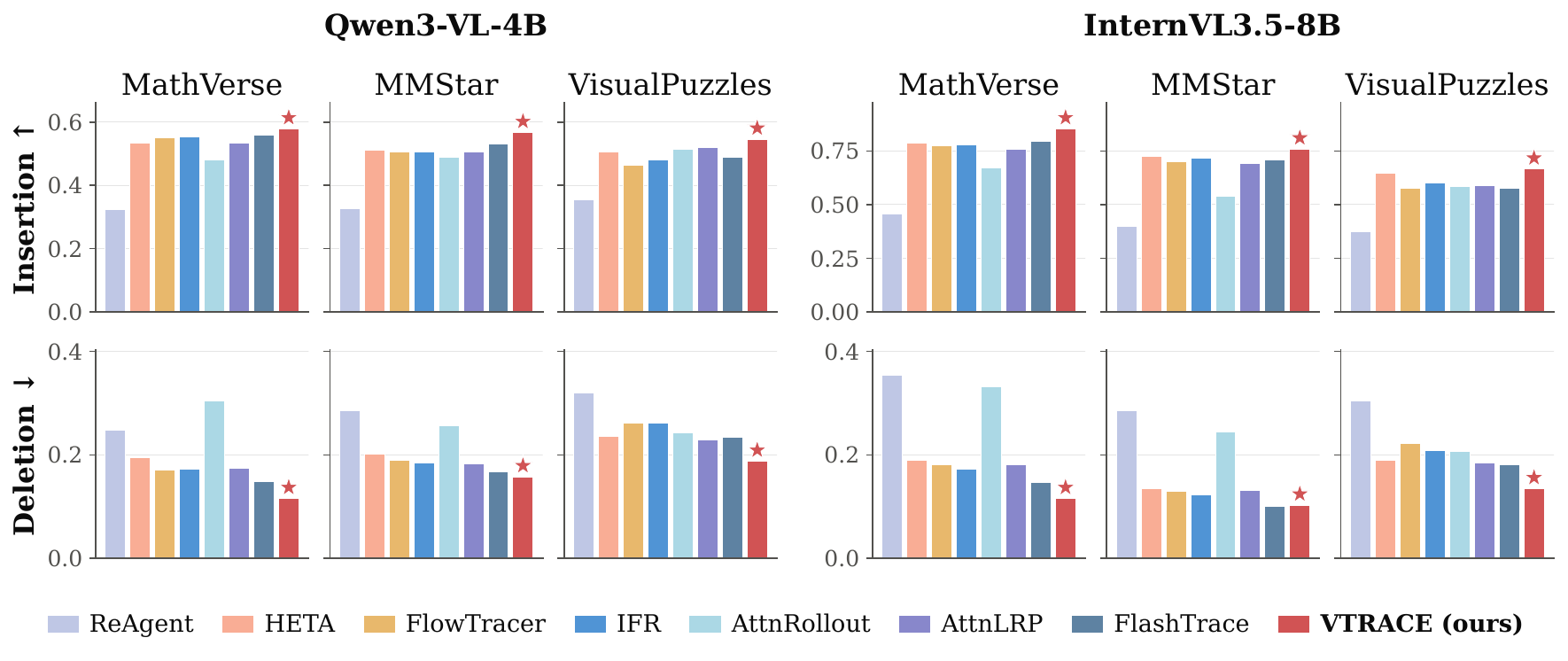}
    \caption{Attribution faithfulness across different LVLM sizes and families. 
    }
    \label{fig:ablation_model_family_size}
    \vspace{-0.3cm}
\end{figure*}
\textbf{\textsc{VTrace} consistently improves attribution faithfulness across different LVLM architectures and sizes}. Shown in~\Cref{fig:ablation_model_family_size}, \textsc{VTrace} achieves the highest RISE insertion and lowest deletion scores on both Qwen3-VL-4B and InternVL3.5-8B across diverse benchmarks, indicating strong generalization across model families and sizes.
Full results are in Appendix~\ref{appendix:generalizability_experiments}.

\begin{table*}[t!] \centering \caption{ Attribution faithfulness on correct and incorrect model predictions over benchmarks. We compare against the strongest baseline for each metric independently.} \label{tab:correctness_analysis}
\scriptsize \setlength{\tabcolsep}{3.5pt}
\resizebox{\textwidth}{!}{
\begin{tabular}{llcccccccc}
\toprule & & \multicolumn{4}{c}{Correct Predictions} & \multicolumn{4}{c}{Incorrect Predictions} \\ \cmidrule(lr){3-6} \cmidrule(lr){7-10} Variant & Method & Del. RISE $\downarrow$ & Del. MAS $\downarrow$ & Ins. RISE $\uparrow$ & Ins. MAS $\uparrow$ & Del. RISE $\downarrow$ & Del. MAS $\downarrow$ & Ins. RISE $\uparrow$ & Ins. MAS $\uparrow$ \\ 
\midrule 
\multirow{2}{*}{Image} & \textbf{Best Baseline} & 0.377 & 0.508 & 0.590 & 0.451 & 0.424 & 0.579 & 0.603 & 0.460 \\ 
& \textbf{\textsc{VTrace}} & \textbf{0.333} & \textbf{0.468} & \textbf{0.644} & \textbf{0.514} & \textbf{0.398} & \textbf{0.550} & \textbf{0.636} & \textbf{0.496} \\ 
\midrule \multirow{2}{*}{Joint} & \textbf{Best Baseline} & 0.209 & 0.338 & 0.581 & 0.381 & 0.200 & 0.323 & 0.605 & 0.410 \\ 
& \textbf{\textsc{VTrace}} & \textbf{0.171} & \textbf{0.235} & \textbf{0.626} & \textbf{0.463} & \textbf{0.176} & \textbf{0.241} & \textbf{0.640} & \textbf{0.474} \\ 
\bottomrule \end{tabular} 
}
\vspace{-0.2cm}
\end{table*}

\begin{table*}[t!]
\centering
\caption{
Attribution-guided learning with \textsc{VTRACE} on Qwen3-VL-4B. Incorporating \textsc{VTrace} attribution into GRPO improves the average performance across benchmarks.
}
\label{tab:training_analysis}
\small
\resizebox{\textwidth}{!}{
\begin{tabular}{lccccccccc}
\toprule
Method
& MMStar
& MathVista
& MMMU
& MathVerse
& MMMU-Pro
& HallusionBench
& RealWorldQA
& MathVision
& Avg. \\
\midrule

\textbf{Base}
& 63.53
& 73.30
& 55.33
& 56.55
& 47.57
& 70.87
& 72.55
& 41.78
& 60.19 \\

\textbf{GRPO}
& 69.53
& \textbf{77.40}
& 61.89
& 62.64
& \textbf{51.68}
& 72.77
& 72.29
& 43.75
& 63.99 \\

+ \textbf{\textsc{VTrace}}
& \textbf{70.20}
& 77.20
& \textbf{65.00}
& \textbf{65.13}
& 51.45
& \textbf{73.29}
& \textbf{73.20}
& \textbf{46.71}
& \textbf{65.27} \\

\bottomrule
\end{tabular}
}
\vspace{-0.5cm}
\end{table*}
\textbf{\textsc{VTrace} improves attribution faithfulness regardless of answer correctness.} Shown in~\Cref{tab:correctness_analysis}, \textsc{VTrace} consistently outperforms the strongest baselines across different evaluation settings. This indicates that \textsc{VTrace} can faithfully identify influential token relationships regardless of the model's prediction correctness. Full results are in Appendix~\ref{appendix:correctness_full_results}.

\begin{wrapfigure}{r}{0.47\textwidth}
  \vspace{-0.6cm}
  \centering
  \includegraphics[width=\linewidth]{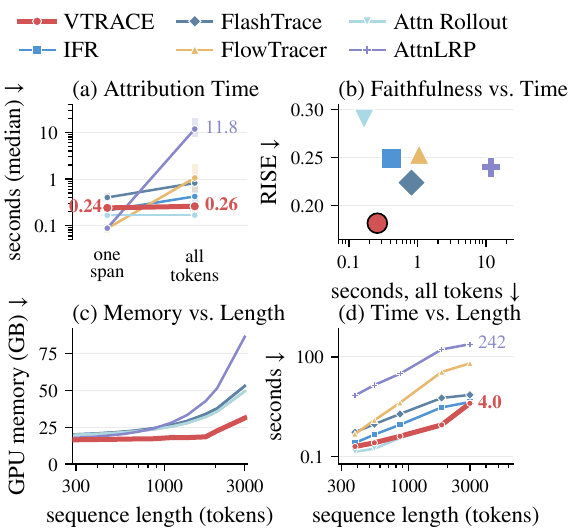}
  \vspace{-0.3cm}
  \caption{Efficiency analysis.}
  \label{fig:efficiency}
  \vspace{-1.0cm}
\end{wrapfigure}
\textbf{\textsc{VTrace} can also provide a meaningful learning signal for improving LVLM reasoning}.
Following~\citep{flowtracer}, we use its token-level attribution for credit assignment in Group Relative Policy Optimization (GRPO), directly integrating traced reasoning flow into post-training. Shown in~\Cref{tab:training_analysis}, attribution-guided GRPO improves the average performance of Qwen3-VL-4B from $63.99$ to $65.27$. This demonstrates that \textsc{VTrace} captures token relationships that are useful for optimization, extending its value beyond interpretability to downstream reinforcement learning and post-training.

\begin{figure*}[t]
    \centering
    \includegraphics[width=0.97\linewidth]{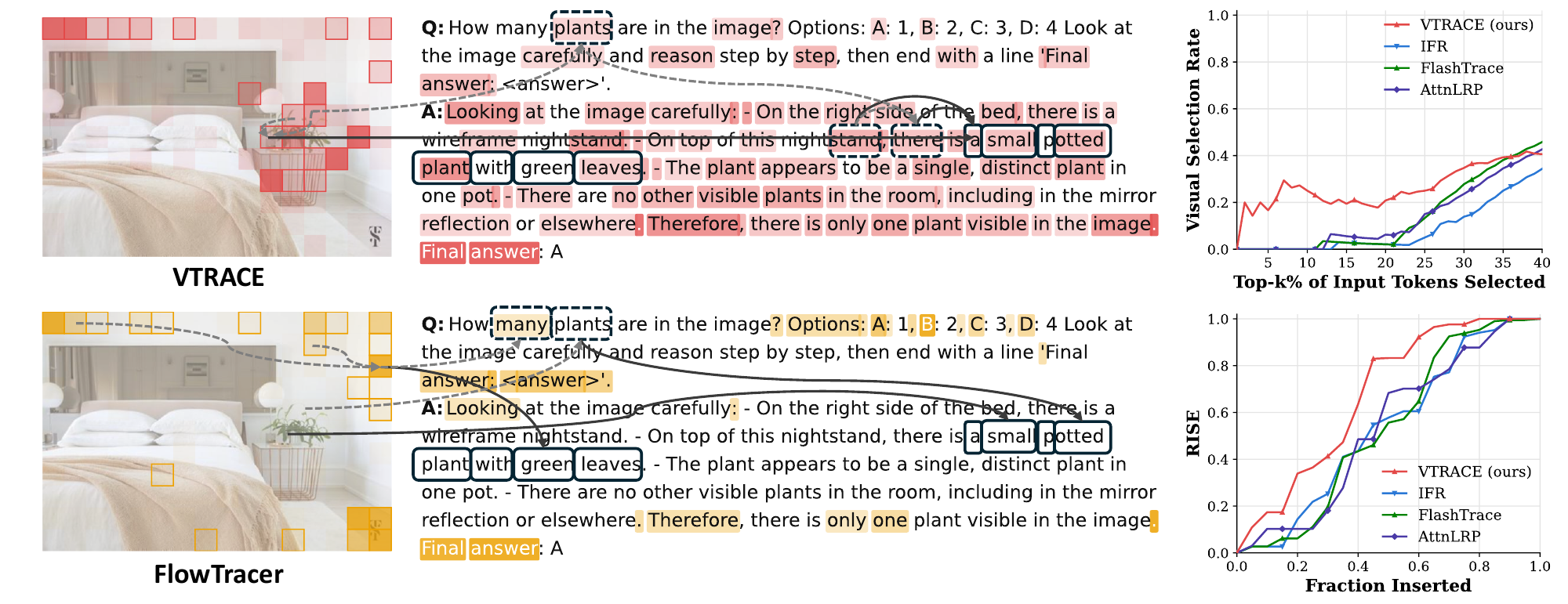}

    \caption{Qualitative study: \textsc{VTrace} better traces reasoning tokens to key visual evidence.}
    \label{fig:case_study} 
\end{figure*}

\begin{figure*}[t]
    \centering
    \includegraphics[width=0.95\linewidth]{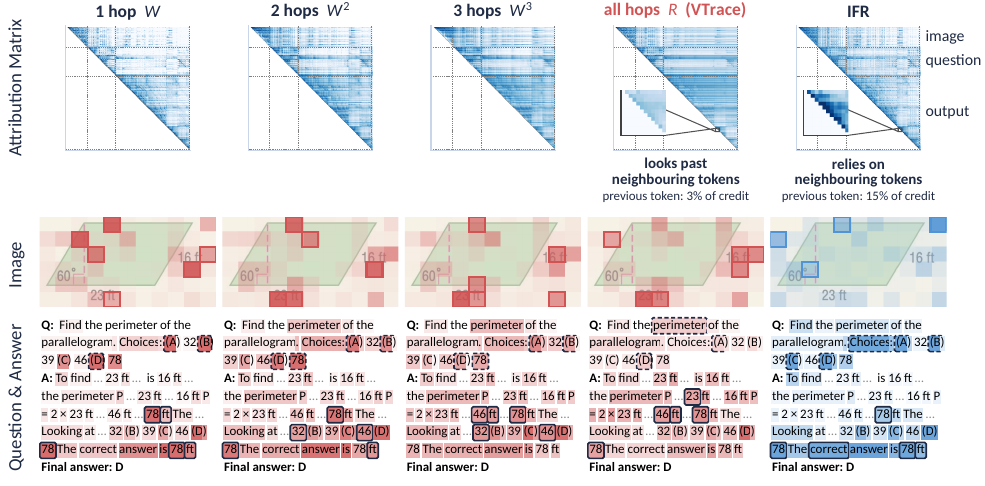}
    \caption{Qualitative study: Attribution comparison between diverse hops and IFR.
}
    \label{fig:hops_casestudy} 
    \vspace{-0.3cm}
\end{figure*}

\vspace{-0.2cm}
\subsection{Qualitative Analysis} \label{sec:qualitative_analysis}
\vspace{-0.1cm}
\textbf{\textsc{VTrace} achieves faster recovery under insertion and sharper degradation under deletion.} 
Elaborated in Appendix~\ref{appendix:fine_grained_plot}~\Cref{fig:fine_grained_curves}, \textsc{VTrace} achieves faster recovery under insertion and sharper degradation under deletion. 

\textbf{Case study.}
\textbf{\textsc{VTrace} more clearly traces visual evidence through intermediate reasoning to the final response.} Evident in~\Cref{fig:case_study}, \textsc{VTrace} assigns stronger attribution to plant-related image patches, while the baseline focuses more on background regions. \textsc{VTrace}'s token-level paths further connect plant-related reasoning tokens back to the corresponding visual patches, which the baseline largely misses.
Notably, when top 10\% tokens are recovered, image tokens account for around a quarter of \textsc{VTrace}'s selected tokens, while baselines select almost none. The RISE insertion curve (bottom-right) further shows that cross-modal calibration helps \textsc{VTrace} recover response likelihood with fewer inserted tokens.
Moreover, \Cref{fig:hops_casestudy} shows that limited-hop attribution concentrates credit on nearby reasoning tokens and misses key visual evidence (e.g., IFR), reiterating the challenge we identified in Sec~\ref{subsec:limited_hops_challenge}. By aggregating all reasoning paths, \textsc{VTrace} traces this credit back to the image patches that support the answer. An extended analysis of the diverse hops and \textsc{VTrace} is in Appendix~\ref{appendix:hops_ablation_performance}. Additional case studies are in Appendix~\ref{appendix:case_studies}.

\vspace{-0.2cm}
\section{Conclusion}
\vspace{-0.2cm}
In this work, we investigate faithful multimodal token attribution for LVLM reasoning, with the aim of tracing generated predictions back to the faithful visual and textual evidence that supports the generation.
Through empirical analysis, we find that existing methods struggle with contributions propagated through intermediate reasoning and underrepresent relevant visual evidence.
To address these challenges, we propose \textsc{VTrace}, which represents token attributions in a pairwise attribution matrix and considers both direct and indirect paths. \textsc{VTrace} further calibrates image and text attribution to provide an unbiased cross-modal ranking.
Experiments across six visual reasoning benchmarks consistently outperforms baseline methods, demonstrating that \textsc{VTrace} provides more faithful multimodal attribution while maintaining efficiency.

\bibliography{iclr2027_conference}
\bibliographystyle{iclr2027_conference}

\newpage
\appendix
\section*{Appendix}

\noindent We summarize the key structure of the Appendix as follow:

\begin{itemize}
    \item \textbf{\Cref{appendix:limitations}:} Limitations
    \item \textbf{\Cref{appendix:extended_related_work}:} Extended Related Work
    \item \textbf{\Cref{proof:W_closed_form}:} Proof of \Cref{eq:katz_closed_form}
    \item \textbf{\Cref{appendix:algo}:} \textsc{VTrace} Algorithm
    \item \textbf{\Cref{appendix:baseline_and_benchmarks}:} Baseline \& Benchmark Descriptions
    \item \textbf{\Cref{appendix:experimental_details}:} Experimental Details
    \item \textbf{\Cref{appendix:evaluaton_details}:} Evaluation Details
    \item \textbf{\Cref{appendix:extended_experiments}:} Extended Experiments
    \item \textbf{\Cref{appendix:case_studies}:} Case Studies
\end{itemize}

\section{Limitations} \label{appendix:limitations}
While \textsc{VTrace} demonstrates strong effectiveness, several limitations remain.
Our current evaluation focuses primarily on image-based visual reasoning, while modern LVLMs also operate in broader settings, including multi-image reasoning, long-context multimodal understanding, and longer-horizon agentic interaction. Extending \textsc{VTRACE} to these settings would further establish the generality of the proposed attribution framework.
Another direction is to integrate attribution into model training to improve model visual reasoning behavior. Our preliminary attribution-guided learning experiments suggest that this direction is promising.

\section{Extended Related Work}
\label{appendix:extended_related_work}
\subsection{Token Attribution and Information Flow}

\textbf{Token attribution} traces how source tokens contribute to the generation of a target token or output span. Existing methods differ mainly in the model signals used to define this contribution. Attention Rollout composes attention matrices across layers to capture how information is mixed between token positions throughout the network~\citep{attn_rollout}. ALTI similarly aggregates token interactions across layers, but derives them from the attention block while accounting for residual connections and layer normalization~\citep{alti}. Chefer et al.~\citep{Chefer_2021_ICCV} propose to propagate relevance through attention and residual operations, while AttnLRP extends the layer-wise relevance propagation to Transformer attention and supports attribution to both input tokens and intermediate representations~\citep{attnLRP}. HETA incorporates attention and value information together with Hessian-based sensitivity and KL divergence under token masking~\citep{HETA}.

\textbf{For generative models}, attribution must additionally account for the role of previously generated tokens. A source token may affect the final prediction directly or indirectly through intermediate tokens in the generated sequence. ALTI+ study this setting in machine translation by tracing contributions from both the source sentence and the generated prefix~\citep{ALTI+}. IFR represents each prediction as a computation graph whose nodes correspond to token states and model components, allowing attribution to follow routes through internal representations~\citep{IFR}.

Several recent approaches \textbf{explicitly model indirect contribution} through generated tokens. FlashTrace aggregates attribution over output spans and recursively follows attribution absorbed by intermediate generated tokens through limited steps~\citep{flashtrace}. This captures cases in which an early source influences the selected output through later reasoning tokens rather than through a single direct interaction. FlowTracer constructs an attention-based flow graph, assigns capacities to edges, reweights them according to their ability to reach the target region, and imposes local flow conservation~\citep{flowtracer}. The resulting token-level flow scores are further used as feedback for reinforcement learning. Both methods therefore move beyond purely local attribution by considering routes through intermediate tokens. \textsc{VTrace} instead aggregates multiple forward paths between source and receiver positions, allowing direct and multi-step effects to be represented within a single computation. The resulting incoming and outgoing contributions are then used to score earlier tokens with respect to selected receiver tokens.

A \textbf{complementary family} of methods \textbf{estimates importance through perturbation or intervention} rather than tracing internal connections. ReAGent repeatedly replaces subsets of input tokens and updates their importance according to changes in the model's next-token prediction~\citep{reagent}. This requires only forward evaluations and does not depend on gradients or explicit access to internal attribution signals. \textsc{VTrace} separates this role from its token-level attribution computation: the pairwise attribution matrix is obtained once from cached model states, while additional forward evaluations with replaced image or prompt embeddings are used only to calibrate the relative contribution of the two input modalities. In our experiments, we further evaluate attribution faithfulness by perturbing the input tokens ranked by each method and measuring the resulting change in model prediction.

\subsection{Attribution Across Image and Text}

Attribution in large vision-language models introduces an additional challenge because evidence is distributed across heterogeneous visual and textual representations. A meaningful attribution must therefore identify important evidence within each modality fairly. Chefer et al.~\citep{Chefer_2021_ICCV} extend Transformer explanation methods to multimodal architectures containing self-attention, co-attention, and encoder-decoder attention, propagating relevance through the attention interactions that connect visual and textual representations. LVLM-Interpret provides attention maps, relevance maps, and causal visualizations for inspecting how large vision-language models generate responses~\citep{lvlm_interpret}. GLIMPSE combines gradient-weighted attention with propagation across layers and aggregation over generated tokens to identify visual and textual evidence supporting a generated response~\citep{glimpse}.

\subsection{Attribution for Reasoning LVLM}

Recent LVLMs increasingly generate explicit intermediate reasoning before producing a final answer. Multimodal-CoT generates a textual rationale conditioned on both image and language inputs and then uses this rationale to predict the response~\citep{multimodal_cot}. LLaVA-CoT structures multimodal reasoning into stages including summarization, visual interpretation, logical reasoning, and conclusion, and performs search over intermediate stages during inference~\citep{llava_cot}. These approaches make intermediate generated tokens an explicit part of the inference process rather than treating the model as producing only a final answer.

Prior work has also examined whether multimodal reasoning faithfully reflects the evidence used by the model. Chen et al.~\citep{cot_vlm} introduce a benchmark and metrics for evaluating reasoning consistency in vision-language models. Balasubramanian et al.~\citep{vlm_cot_faithfulness} study whether generated reasoning reflects visual and textual biases introduced into the input. Their results show that models are less likely to acknowledge subtle visual cues than explicit textual cues, highlighting a potential mismatch between generated explanations and the evidence that actually affects the prediction. Such findings motivate attribution methods that inspect model-internal contributions rather than relying only on the content of the generated rationale.

\textsc{VTrace} targets this setting by estimating how information propagates through multimodal generations. Given selected receiver tokens, which may correspond to reasoning steps, answer tokens, or an output span, it traces contributions from earlier image, prompt, and generated tokens through both direct and indirect paths. This provides a way to inspect which input tokens supports an intermediate reasoning step, how earlier reasoning tokens influence later ones, and which sources ultimately contribute to the final answer.

\section{Proof of Eq.~\ref{eq:katz_closed_form}} \label{proof:W_closed_form}
\begin{theorem}[Closed-form multi-path aggregation]
Let $\widehat W \in \mathbb{R}^{T \times T}$ be the strictly upper-triangular matrix. Define:
\begin{equation}
R = \sum_{\tau=1}^{T-1} \gamma^\tau \widehat W^\tau.
\end{equation}
Then:
\begin{equation}
R = (\mathbf I - \gamma \widehat W)^{-1} - \mathbf I.
\end{equation}
\end{theorem}

\begin{proof}

Because $\widehat W$ is strictly upper triangular,
$\widehat W^T=\mathbf 0$.
Now consider the finite matrix geometric series and its expansion:
\begin{align}
S=\sum_{\tau=0}^{T-1}(\gamma\widehat W)^\tau
&=\mathbf I
+\gamma\widehat W
+\gamma^2\widehat W^2
+\cdots
+\gamma^{T-1}\widehat W^{T-1}\\
&=\mathbf I 
+ \sum_{\tau=1}^{T-1} \gamma^\tau\widehat W^\tau. 
\end{align}

Multiplying $S$ by $\mathbf I-\gamma\widehat W$ gives:
\begin{equation}
(\mathbf I-\gamma\widehat W)S = S-\gamma\widehat W S.
\end{equation}

Substituting the expression for $S$, we have:
\begin{align}
(\mathbf I-\gamma\widehat W)S &=
\left( \mathbf I
+\gamma\widehat W
+\gamma^2\widehat W^2
+\cdots
+\gamma^{T-1}\widehat W^{T-1}
\right)
\nonumber\\
&\quad -
\left(
\gamma\widehat W
+\gamma^2\widehat W^2
+\cdots
+\gamma^T\widehat W^T \right).
\end{align}

Note that all the matching intermediate terms cancel, which leaves us with:
\begin{equation}
(\mathbf I-\gamma\widehat W)S = \mathbf I-\gamma^T\widehat W^T.
\end{equation}

Since $\widehat W^T=\mathbf 0$, this reduces to:
\begin{equation} 
(\mathbf I-\gamma\widehat W) S = \mathbf I. 
\end{equation} 

Hence, we have:
\begin{equation} 
 S = \mathbf I+\sum_{\tau=1}^{T-1} \gamma^\tau\widehat W^\tau = (\mathbf I-\gamma\widehat W)^{-1}.
\end{equation} 

Finally, subtracting $\mathbf I$ from both sides gives: 
\begin{equation} 
R = \sum_{\tau=1}^{T-1}\gamma^\tau\widehat W^\tau = (\mathbf I-\gamma\widehat W)^{-1}-\mathbf I,
\end{equation} 
as required.
\end{proof}

\section{\textsc{VTrace} Algorithm} \label{appendix:algo}
\textbf{Algorithm~\ref{alg:VTrace} gives the full procedure in pseudocode.} Its input is the cached state of a single forward pass over the frozen trace: the value vectors, the attention weights, and the residual stream before and after each attention block. 
\begin{itemize}
 \item Step 1 computes the modality-aware pairwise attribution matrix $W$ from the cached model states, as described in Section~\ref{sec:direct_attribution}.
 \item Step 2 normalizes $W$ to obtain $\widehat W$, computes the multi-hop attribution matrix $R$ in closed form, and scores each source token using its incoming attribution and outgoing attribution to the selected receivers, as described in Section~\ref{sec:multihop}.
 \item Step 3 is the modality calibration of Section~\ref{sec:modality_alignment}: three further forward passes measure the loss increase when the image, the text, or both are replaced by padding, the two modality weights are the Shapley values of this two player game, and the image scores are rescaled so that their share of the total matches the image weight while every within-modality order is left unchanged. This ensures the image and text tokens are well calibrated to be ranked fairly and faithfully.
\end{itemize}

\textbf{Calibration boundary conditions.}
Calibration is applied over image and query token positions when used in perturbation evaluation.
Calibration is applied only when both modality contributions from~\Cref{eq:modality_shapley} and both sums of token attribution scores are positive.
If any of the quantities is zero or negative, the original scores are maintained.
When applied, calibration multiplies image scores by the positive factor and leaves text scores unchanged.
The operation preserves the ranking within each modality, and matches their total attribution ratio to the estimated contribution ratio.

\begin{algorithm}[H]
\caption{\textsc{VTrace}}
\label{alg:VTrace}
\begin{lstlisting}[style=pyalg]
# Input:  cached states of one forward pass (V, a, x_attn, x_pre),
#         attention output projections Wo, receiver positions recv_pos,
#         image positions img_pos, user-text positions txt_pos, gamma
# Output: VTrace attribution u_VTrace

# 1. Compute the modality-aware pairwise attribution matrix
F[l,i,h] = V[l,i,h] @ Wo[l,h]        # per-head output of i
d[l,j]   = x_attn[l,j] - x_pre[l,j]

# center within modality for input tokens (image or prompt-text tokens)
F[l,i,h] -= mean_k(F[l,k,h], k in modality(i)), i in img_pos or txt_pos

e[l,i,j] = sum_h(a[l,h,j,i] * dot(F[l,i,h], d[l,j]))
W[i,j]   = mean_l(relu(e[l,i,j]) / (norm(d[l,j]) + eps)), i < j

# 2. Katz-based multi-hop attribution
W_hat = W / (max(W) + eps)
R = inv(I - gamma * W_hat) - I
inflow[r] = sum_i(R[i,r])
outflow[r] = sum(R[r,j] for j in recv_pos)
u[r] = (1 + inflow[r]) * outflow[r]
u[min(recv_pos):] = 0                # sources must precede the receivers

# 3. Align image and text attribution scores
# damage(pos) = logp(gen | clean) - logp(gen | pad at pos)
D_img, D_txt = damage(img_pos), damage(txt_pos)
D_both       = damage(img_pos + txt_pos)
phi_img = (D_img + D_both - D_txt) / 2
phi_txt = (D_txt + D_both - D_img) / 2
s = phi_img / (phi_img + phi_txt)     # image share
lam = (s / (1 - s)) * (sum(u[txt_pos]) / sum(u[img_pos]))

# calibrate the image contribution of the attribution
u_VTrace = u
u_VTrace[img_pos] = lam * u[img_pos] 
return u_VTrace

\end{lstlisting}
\end{algorithm}

\section{Baseline \& Benchmark Descriptions}
\label{appendix:baseline_and_benchmarks}

\subsection{Baseline Methods}
We compare \textsc{VTrace} with seven attribution methods:

\begin{itemize}
    \item \textbf{ReAGent}~\citep{reagent} is a perturbation method, which repeatedly replaces input tokens with plausible alternatives drawn from a masked language model and scores each token by the change this causes in the probability of the generated text.
    \item \textbf{HETA}~\citep{HETA} is a gradient method that adds second order terms from the Hessian to first order gradient attribution, so that interactions between input tokens contribute to their scores. We follow the released implementation.
    \item \textbf{Attention Rollout}~\citep{attn_rollout} composes head-averaged attention across layers. At each layer, attention is combined with the identity matrix as $0.5A^{(\ell)}+0.5I$ and row-normalized. The resulting layer products provide cumulative input-to-answer attribution.
    \item \textbf{AttnLRP}~\citep{attnLRP} applies layer-wise relevance propagation through attention blocks with relevance-conserving rules for softmax and attention matrix multiplication. It requires one backward pass per attributed target token, so its cost increases with response length.
    \item \textbf{IFR}~\citep{IFR} builds on ALTI~\citep{alti} to decompose the Transformer into token-level information flows. It measures how each token contributes to subsequent representations through attention and MLP blocks, tracing these contributions layer by layer.
    \item \textbf{FlowTracer}~\citep{flowtracer} constructs a flow network from head-averaged attention over a selected layer range and scores each input token by the flow it carries to the answer tokens.
    \item \textbf{FlashTrace}~\citep{flashtrace} uses span-wise aggregation and recursive attribution to trace importance through intermediate reasoning. At each hop, important reasoning tokens become weighted targets for the next attribution step, propagating influence backward toward the original input. Attribution across hops is then aggregated into the final input scores.
\end{itemize}

\subsection{Benchmark Statistics}
We evaluate our method on diverse multimodal benchmarks. These benchmarks cover a broad range of multimodal capabilities, ranging from general visual perception to knowledge understanding and mathematical problem solving.
A brief introduction for each benchmark is provided below.

\begin{itemize}
    \item \textbf{MMStar}~\citep{mmstar} is a vision-indispensable benchmark designed to evaluate multimodal capabilities of LVLMs while reducing the text-only shortcuts effect. It contains 1500 samples covering six core capabilities: coarse perception, fine-grained perception, instance reasoning, logical reasoning, science and technology, and mathematics.
    \item \textbf{MathVista}~\citep{mathvista} is a benchmark for mathematical reasoning in visual contexts. Its questions cover figure question answering, geometry problem solving, math word problems, textbook question answering and visual question answering, and span seven reasoning types: algebraic, arithmetic, geometric, logical, numeric commonsense, scientific and statistical reasoning. We use the \emph{testmini} split of 1,000 samples.
    \item \textbf{MathVerse}~\citep{mathverse} evaluates visual mathematical reasoning using problems from plane geometry, solid geometry, and functions. Each problem is provided in six variants with different amounts of textual and visual information. We use the \emph{testmini} split.
    \item \textbf{MMMU}~\citep{MMMU} consists of college level multimodal questions collected from exams, quizzes and textbooks. It spans six disciplines: art and design, business, science, health and medicine, humanities and social science, and technology and engineering, across 30 subjects and 30 image types such as charts, diagrams, tables, chemical structures and medical images. We use the validation split of 900 samples.
    \item \textbf{MMMU-Pro}~\citep{mmmu-pro} extends MMMU to provide a more challenging evaluation of multimodal reasoning. It removes questions that text-only models answer correctly, expands the candidate options, and adds a vision-only setting in which the question is embedded in the image, so that answering requires reading the image.
    \item \textbf{VisualPuzzles}~\citep{visualPuzzles} is a benchmark that decouples multimodal reasoning from domain knowledge. It contains 1,168 puzzles adapted from logical reasoning questions of the Chinese Civil Service Examination, covering five reasoning categories: algorithmic, analogical, deductive, inductive and spatial reasoning.
\end{itemize}
We additionally test three visual reasoning benchmarks for our post-training experiments.
\begin{itemize}
    \item \textbf{HallusionBench}~\citep{hallusionbench} evaluates hallucination and visual reasoning in LVLMs using 346 images and 1,129 human-designed questions. It particularly emphasis on failures caused by visual illusion and language hallucination.
    
    \item \textbf{RealWorldQA}~\citep{realworldqa} evaluates multimodal understanding in real-world scenes. The benchmark includes anonymized vehicle-view images and other natural scenes, with questions focusing particularly on spatial relationships, object properties, and everyday visual understanding.
    
    \item \textbf{MathVision}~\citep{mathvision} evaluates visual mathematical reasoning using problems from real mathematics competitions. It spans 16 mathematical disciplines and five difficulty levels, requiring models to jointly interpret visual content and perform advanced mathematical reasoning.
\end{itemize}

\section{Experimental Details}
\label{appendix:experimental_details}
\textbf{Models.} We evaluate on the publicly released Qwen3-VL~\citep{qwen3vl} checkpoints at 8B and 4B parameters and on InternVL3.5-8B~\citep{internvl_3.5}.

\textbf{Reasoning Trace Generation.} The benchmarks provide the image, the question and the reference answer, but not the reasoning trace that we attribute. We therefore run each model once per question with greedy decoding and a budget of 2,048 max tokens, and record the full context: the image tokens, the question, the generated reasoning and the final answer. This trace is then frozen. Every attribution method is evaluated on the same trace, and every perturbation in the evaluation measures the likelihood of the same trace, so this avoids any discrepancy caused by the generated reasoning traces. For \textsc{VTrace}, the decay rate $\gamma = 1$ by default. We highlight that we do not filter traces by answer correctness, and Table~\ref{tab:correctness_analysis} reports the metrics separately for correct and incorrect predictions.

\textbf{Hardware.} All token attribution-based experiments run on a single NVIDIA RTX PRO 6000 Blackwell GPU with 96GB of memory. Post-training based experiments (e.g., GRPO) are conducted on AMD Instinct MI355X GPUs.

\subsection{Attribution-guided Learning Details}
\label{appendix:attribution_guided training}
We conduct the attribution-guided experiments~\Cref{sec:experiments} using the \texttt{verl} framework~\cite{verl}. Specifically, we fine-tune Qwen3-VL-4B-Instruct with Group Relative Policy Optimization (GRPO).
For each training sample, we generate $n=8$ rollout responses with a temperature of $0.7$ and top-$p=0.95$.

The training data are constructed by combining and filtering publicly available multimodal reasoning data from MMR1-RL~\cite{mmr1} and the reinforcement-learning split of ReVisual-R1~\cite{revisualr1}.
We follow the targeted-RL weighting strategy of~\cite{flowtracer}, where we use \textsc{VTrace} attribution scores to
identify the most important response tokens. The top $40\%$ of response tokens are assigned a weight of $1.5$, while all remaining tokens retain the default weight of $1.0$. All models are trained for 2 epochs under the same training configuration.

\section{Evaluation Details}
\label{appendix:evaluaton_details}
Because there are no token-level ground-truth labels indicating which image regions or question tokens contribute to a generated reasoning trace, we evaluate attribution using perturbation-based metrics. Following prior attribution work, we use RISE~\citep{rise} and MAS~\citep{MAS}, considering both their deletion and insertion variants. These metrics progressively remove or restore input tokens according to their attribution scores and measure how the model's confidence in the generated trace changes.

\subsection{Perturbation Protocol}
\textbf{Deletable Tokens.} We perturb only tokens from the multimodal input. Let $\mathcal{I}$ and $\mathcal{T}$ denote the image-token and question-token positions, respectively. We consider two evaluation settings: \textbf{image}, where only tokens in $\mathcal{I}$ are perturbed, and \textbf{joint}, where tokens in both $\mathcal{I}$ and $\mathcal{T}$ are perturbed. Chat-template tokens are excluded, and generated tokens are kept fixed because they constitute the reasoning trace being explained.

\textbf{Perturbation schedule.}
For image tokens, perturbation is performed in pixel space. We first construct a Gaussian-blurred version of the image and replace the image regions corresponding to selected visual tokens with their blurred counterparts. The modified image is re-encoded by the LVLM at each perturbation step. For question tokens, selected tokens are replaced with the padding token.

Moreover, perturbing one token at a time would require a separate forward pass for every token. Following FlashTrace~\citep{flashtrace}, we therefore evaluate attribution in proportional steps. Given attribution scores $u$, the $P$ deletable tokens are ranked by decreasing attribution magnitude $|u_i|$ and partitioned into $K=\min(20,P)$ approximately equal-sized groups. Each evaluation thus requires $K+1$ forward passes.
\begin{itemize}
    \item For \textbf{deletion}, evaluation starts from the original input and progressively perturbs groups from highest to lowest attribution. A faithful attribution should therefore cause the \textbf{model response to decrease rapidly}.
    \item For \textbf{insertion}, evaluation starts from the fully perturbed input and progressively restores the same groups in the same order. A faithful attribution should \textbf{recover the model response rapidly}.
\end{itemize}

\subsection{Metric Definitions}
Let
$$ f(X) = \exp\left( \frac{1}{n_{\mathrm{gen}}} \sum_{t=1}^{n_{\mathrm{gen}}} \log p_\theta(y_t\mid X,y_{<t}) \right) $$
denote the length-normalised likelihood of the original generated trace under context $X$, as defined in Section $\ref{sec:implementation_details}$.

Let $\pi$ denote the ordering of the $P$ deletable tokens by decreasing attribution magnitude $|u_i|$, with attribution scores $u$. We define $X_{\mathrm{del}}^{(k)}$ as the input after perturbing the first $k$ groups under this ordering, and $X_{\mathrm{ins}}^{(k)}$ as the fully perturbed input after restoring the first $k$ groups, for $k=0,\ldots,K$. We have:
$$
f_k^{\mathrm{del}}
=
f\left(X_{\mathrm{del}}^{(k)}\right),
\qquad
f_k^{\mathrm{ins}}
=
f\left(X_{\mathrm{ins}}^{(k)}\right).
$$

We normalize each response curve between its fully perturbed and clean endpoints and enforce its expected monotonic direction:

$$
r_k^{\mathrm{del}}
=
\min_{j\leq k}
\frac{
f_j^{\mathrm{del}}-f_K^{\mathrm{del}}
}{
f_0^{\mathrm{del}}-f_K^{\mathrm{del}}
},
\qquad
r_k^{\mathrm{ins}}
=
\max_{j\leq k}
\frac{
f_j^{\mathrm{ins}}-f_0^{\mathrm{ins}}
}{
f_K^{\mathrm{ins}}-f_0^{\mathrm{ins}}
}.
$$

This allows the normalized responses to be bounded to $[0,1]$. We can define our metrics as:

\subsubsection{RISE}
RISE~\citep{rise} measures the area under the normalized perturbation curve:
$$
\mathrm{RISE}_{\mathrm{del}}
=
\mathrm{AUC}\left(r^{\mathrm{del}}\right)
\downarrow,
\qquad
\mathrm{RISE}_{\mathrm{ins}}
=
\mathrm{AUC}\left(r^{\mathrm{ins}}\right)
\uparrow.
$$
A faithful attribution should remove important evidence early under deletion, causing the response to decrease rapidly, and restore it early under insertion, causing the response to recover rapidly. Thus, \textbf{lower deletion RISE} and \textbf{higher insertion RISE} indicate \textbf{better faithfulness}. Notably, RISE depends only on the attribution ranking and does not account for attribution magnitude.

\subsubsection{MAS}
Magnitude Aligned Scoring (MAS)~\citep{MAS} additionally evaluates whether attribution magnitude agrees with the observed model response. Let $\mathcal{G}_k$ denote the token group modified at step $k$. The fraction of attribution mass restored after insertion step $k$ is given by:

$$
m_k^{\mathrm{ins}}
=
\frac{
\sum_{j=1}^{k}
\sum_{i\in\mathcal{G}_j}|u_i|
}{
\sum_i |u_i|
},
$$
while the mass remaining after deletion is
$
m_k^{\mathrm{del}}
=
1-m_k^{\mathrm{ins}}.
$
The mismatch between attribution mass and model response is thus given by:
$$
d_k^{\mathrm{del}}
=
\left|
r_k^{\mathrm{del}}-m_k^{\mathrm{del}}
\right|,
\qquad
d_k^{\mathrm{ins}}
=
\left|
r_k^{\mathrm{ins}}-m_k^{\mathrm{ins}}
\right|.
$$

MAS incorporates this mismatch into the corresponding perturbation curve:

$$
\mathrm{MAS}_{\mathrm{del}}
=
\mathrm{AUC}\left(
r^{\mathrm{del}}+d^{\mathrm{del}}
\right)
\downarrow,
\qquad
\mathrm{MAS}_{\mathrm{ins}}
=
\mathrm{AUC}\left(
r^{\mathrm{ins}}-d^{\mathrm{ins}}
\right)
\uparrow.
$$

MAS therefore considers both attribution ranking and magnitude, where a faithful attribution should assign attribution mass in proportion to the observed change in model response. As with RISE, \textbf{lower deletion} and \textbf{higher insertion scores} indicate \textbf{better faithfulness}. In implementation, the normalized response and MAS curves are clipped to $[0,1]$, and AUC is computed using the trapezoidal rule.

\subsection{System Prompt}
A fixed system prompt is used to generate reasoning traces. Specifically, the system prompt instructs the model to reason step by step over the visual evidence before providing the final answer.

\begin{promptbox}{Prompt Template for Reasoning Trace Generation}

\textbf{System:}

\smallskip
Look at the image carefully and reason step by step, then end with a line 'Final answer: <answer>'.

\medskip
\textbf{User:}

\smallskip
\texttt{\{image\}} \texttt{\{question\}}

\end{promptbox}

\section{Extended Experiments} \label{appendix:extended_experiments}
\subsection{Main Results}
\label{appendix:main_results}
\Cref{tab:appendix_image_faithfulness_blur} and~\Cref{tab:appendix_joint_faithfulness} report the complete Image-variant and Joint-variant faithfulness results with both RISE and MAS under insertion and deletion on Qwen3-VL-8B. \textsc{VTrace} consistently outperforms the baselines for both the image and joint variants. \textsc{VTrace} improves MAS insertion scores by 9.9\% and 18.4\% for image and join variant respectively, and reduces the deletion scores by 7.7\% and 29.6\%. The results turn out \textsc{VTrace} manages to identify more faithful tokens that are important for constructing the reasoning.

\begin{table}[H]
\centering
\caption{
Attribution faithfulness of the \textbf{Image variant} on Qwen3-VL-8B across six benchmarks. We report RISE and MAS insertion (\textcolor{teal}{$_{\text{ins}}\uparrow$}, higher is better) and deletion (\textcolor{BurntOrange}{$_{\text{del}}\downarrow$}, lower is better) AUC and misalignment scores. The Image variant perturbs image patch tokens only.
Best results are highlighted in \textcolor{Maroon}{\textbf{red bold}}, and second-best
results are highlighted in \textcolor{NavyBlue}{\underline{blue underlining}}.
}
\label{tab:appendix_image_faithfulness_blur}
\resizebox{\linewidth}{!}{
\renewcommand{\arraystretch}{1.15}
\begin{tabular}{llcccccccc}
\toprule
Dataset & Metric
& \textbf{ReAGent}
& \textbf{HETA}
& \textbf{FlowTracer}
& \textbf{IFR}
& \textbf{Attn Rollout}
& \textbf{AttnLRP}
& \textbf{FlashTrace}
& \cellcolor{blue!10}\textbf{\textsc{VTrace}}\\
\midrule

\multirow{4}{*}{MMStar}
& RISE\textcolor{teal}{$_{\text{ins}}\uparrow$}
& 0.505
& 0.497
& 0.532
& 0.539
& 0.539
& \underline{\textcolor{NavyBlue}{0.563}}
& 0.555
& \cellcolor{blue!10}\textcolor{Maroon}{\textbf{0.600}} \\

& MAS\textcolor{teal}{$_{\text{ins}}\uparrow$}
& 0.336
& 0.329
& 0.368
& 0.384
& 0.399
& 0.392
& \underline{\textcolor{NavyBlue}{0.410}}
& \cellcolor{blue!10}\textcolor{Maroon}{\textbf{0.460}} \\

& RISE\textcolor{BurntOrange}{$_{\text{del}}\downarrow$}
& 0.458
& 0.463
& 0.412
& 0.409
& 0.439
& \underline{\textcolor{NavyBlue}{0.384}}
& 0.394
& \cellcolor{blue!10}\textcolor{Maroon}{\textbf{0.352}} \\

& MAS\textcolor{BurntOrange}{$_{\text{del}}\downarrow$}
& 0.620
& 0.625
& 0.563
& 0.551
& 0.566
& 0.547
& \underline{\textcolor{NavyBlue}{0.528}}
& \cellcolor{blue!10}\textcolor{Maroon}{\textbf{0.489}} \\
\cmidrule(lr){1-10}

\multirow{4}{*}{MathVista}
& RISE\textcolor{teal}{$_{\text{ins}}\uparrow$}
& 0.500
& 0.514
& 0.579
& 0.591
& 0.584
& 0.599
& \underline{\textcolor{NavyBlue}{0.600}}
& \cellcolor{blue!10}\textcolor{Maroon}{\textbf{0.662}} \\

& MAS\textcolor{teal}{$_{\text{ins}}\uparrow$}
& 0.334
& 0.340
& 0.421
& 0.449
& 0.452
& 0.427
& \underline{\textcolor{NavyBlue}{0.464}}
& \cellcolor{blue!10}\textcolor{Maroon}{\textbf{0.535}} \\

& RISE\textcolor{BurntOrange}{$_{\text{del}}\downarrow$}
& 0.447
& 0.444
& 0.368
& 0.363
& 0.394
& 0.358
& \underline{\textcolor{NavyBlue}{0.354}}
& \cellcolor{blue!10}\textcolor{Maroon}{\textbf{0.311}} \\

& MAS\textcolor{BurntOrange}{$_{\text{del}}\downarrow$}
& 0.615
& 0.607
& 0.513
& 0.499
& 0.518
& 0.521
& \underline{\textcolor{NavyBlue}{0.482}}
& \cellcolor{blue!10}\textcolor{Maroon}{\textbf{0.446}} \\
\cmidrule(lr){1-10}

\multirow{4}{*}{MMMU}
& RISE\textcolor{teal}{$_{\text{ins}}\uparrow$}
& 0.543
& 0.564
& 0.595
& 0.610
& \underline{\textcolor{NavyBlue}{0.627}}
& 0.612
& 0.618
& \cellcolor{blue!10}\textcolor{Maroon}{\textbf{0.665}} \\

& MAS\textcolor{teal}{$_{\text{ins}}\uparrow$}
& 0.381
& 0.389
& 0.431
& 0.461
& \underline{\textcolor{NavyBlue}{0.500}}
& 0.444
& 0.479
& \cellcolor{blue!10}\textcolor{Maroon}{\textbf{0.536}} \\

& RISE\textcolor{BurntOrange}{$_{\text{del}}\downarrow$}
& 0.482
& 0.480
& 0.438
& 0.425
& 0.444
& 0.419
& \underline{\textcolor{NavyBlue}{0.417}}
& \cellcolor{blue!10}\textcolor{Maroon}{\textbf{0.379}} \\

& MAS\textcolor{BurntOrange}{$_{\text{del}}\downarrow$}
& 0.639
& 0.643
& 0.606
& 0.576
& 0.580
& 0.582
& \underline{\textcolor{NavyBlue}{0.561}}
& \cellcolor{blue!10}\textcolor{Maroon}{\textbf{0.522}} \\
\cmidrule(lr){1-10}

\multirow{4}{*}{MMMU-Pro}
& RISE\textcolor{teal}{$_{\text{ins}}\uparrow$}
& 0.527
& 0.542
& 0.570
& 0.572
& \underline{\textcolor{NavyBlue}{0.604}}
& 0.593
& 0.583
& \cellcolor{blue!10}\textcolor{Maroon}{\textbf{0.645}} \\

& MAS\textcolor{teal}{$_{\text{ins}}\uparrow$}
& 0.363
& 0.359
& 0.399
& 0.422
& \underline{\textcolor{NavyBlue}{0.468}}
& 0.428
& 0.432
& \cellcolor{blue!10}\textcolor{Maroon}{\textbf{0.511}} \\

& RISE\textcolor{BurntOrange}{$_{\text{del}}\downarrow$}
& 0.472
& 0.472
& 0.442
& 0.435
& 0.444
& \underline{\textcolor{NavyBlue}{0.413}}
& 0.429
& \cellcolor{blue!10}\textcolor{Maroon}{\textbf{0.378}} \\

& MAS\textcolor{BurntOrange}{$_{\text{del}}\downarrow$}
& 0.631
& 0.639
& 0.609
& 0.580
& 0.579
& \underline{\textcolor{NavyBlue}{0.569}}
& 0.575
& \cellcolor{blue!10}\textcolor{Maroon}{\textbf{0.517}} \\
\cmidrule(lr){1-10}

\multirow{4}{*}{MathVerse}
& RISE\textcolor{teal}{$_{\text{ins}}\uparrow$}
& 0.523
& 0.565
& 0.636
& 0.640
& 0.624
& 0.610
& \underline{\textcolor{NavyBlue}{0.644}}
& \cellcolor{blue!10}\textcolor{Maroon}{\textbf{0.684}} \\

& MAS\textcolor{teal}{$_{\text{ins}}\uparrow$}
& 0.359
& 0.380
& 0.475
& 0.509
& 0.498
& 0.445
& \underline{\textcolor{NavyBlue}{0.513}}
& \cellcolor{blue!10}\textcolor{Maroon}{\textbf{0.562}} \\

& RISE\textcolor{BurntOrange}{$_{\text{del}}\downarrow$}
& 0.454
& 0.430
& \underline{\textcolor{NavyBlue}{0.341}}
& 0.345
& 0.384
& 0.360
& \underline{\textcolor{NavyBlue}{0.341}}
& \cellcolor{blue!10}\textcolor{Maroon}{\textbf{0.296}} \\

& MAS\textcolor{BurntOrange}{$_{\text{del}}\downarrow$}
& 0.614
& 0.605
& 0.500
& 0.479
& 0.512
& 0.527
& \underline{\textcolor{NavyBlue}{0.476}}
& \cellcolor{blue!10}\textcolor{Maroon}{\textbf{0.435}} \\
\cmidrule(lr){1-10}

\multirow{4}{*}{VisualPuzzles}
& RISE\textcolor{teal}{$_{\text{ins}}\uparrow$}
& 0.456
& 0.461
& 0.468
& 0.511
& 0.525
& \underline{\textcolor{NavyBlue}{0.531}}
& 0.510
& \cellcolor{blue!10}\textcolor{Maroon}{\textbf{0.557}} \\

& MAS\textcolor{teal}{$_{\text{ins}}\uparrow$}
& 0.296
& 0.262
& 0.272
& 0.337
& \underline{\textcolor{NavyBlue}{0.354}}
& 0.346
& 0.333
& \cellcolor{blue!10}\textcolor{Maroon}{\textbf{0.374}} \\

& RISE\textcolor{BurntOrange}{$_{\text{del}}\downarrow$}
& 0.409
& 0.405
& 0.379
& 0.354
& 0.380
& \underline{\textcolor{NavyBlue}{0.324}}
& 0.358
& \cellcolor{blue!10}\textcolor{Maroon}{\textbf{0.309}} \\

& MAS\textcolor{BurntOrange}{$_{\text{del}}\downarrow$}
& 0.566
& 0.613
& 0.582
& 0.528
& 0.522
& \underline{\textcolor{NavyBlue}{0.489}}
& 0.537
& \cellcolor{blue!10}\textcolor{Maroon}{\textbf{0.456}} \\

\bottomrule
\end{tabular}
}
\end{table}

\begin{table}[H]
\centering
\caption{
Attribution faithfulness of the \textbf{Joint variant} on Qwen3-VL-8B across six benchmarks. The Joint variant perturbs both image and text tokens.
}
\resizebox{\linewidth}{!}{
\renewcommand{\arraystretch}{1.15}
\begin{tabular}{llcccccccc}
\toprule
Dataset & Metric
& \textbf{ReAGent}
& \textbf{HETA}
& \textbf{FlowTracer}
& \textbf{IFR}
& \textbf{Attn Rollout}
& \textbf{AttnLRP}
& \textbf{FlashTrace}
& \cellcolor{blue!10}\textbf{\textsc{VTrace}} \\
\midrule

\multirow{4}{*}{MMStar}
& RISE\textcolor{teal}{$_{\text{ins}}\uparrow$}
& 0.371
& 0.489
& 0.506
& 0.507
& 0.501
& 0.529
& \underline{\textcolor{NavyBlue}{0.554}}
& \cellcolor{blue!10}\textcolor{Maroon}{\textbf{0.581}} \\

& MAS\textcolor{teal}{$_{\text{ins}}\uparrow$}
& 0.165
& 0.277
& 0.299
& 0.305
& 0.314
& 0.314
& \underline{\textcolor{NavyBlue}{0.349}}
& \cellcolor{blue!10}\textcolor{Maroon}{\textbf{0.397}} \\

& RISE\textcolor{BurntOrange}{$_{\text{del}}\downarrow$}
& 0.325
& 0.245
& 0.226
& 0.222
& 0.273
& 0.219
& \underline{\textcolor{NavyBlue}{0.204}}
& \cellcolor{blue!10}\textcolor{Maroon}{\textbf{0.180}} \\

& MAS\textcolor{BurntOrange}{$_{\text{del}}\downarrow$}
& 0.494
& 0.408
& 0.372
& 0.364
& 0.384
& 0.341
& \underline{\textcolor{NavyBlue}{0.332}}
& \cellcolor{blue!10}\textcolor{Maroon}{\textbf{0.246}} \\
\cmidrule(lr){1-10}

\multirow{4}{*}{MathVista}
& RISE\textcolor{teal}{$_{\text{ins}}\uparrow$}
& 0.382
& 0.467
& 0.517
& 0.516
& 0.501
& 0.512
& \underline{\textcolor{NavyBlue}{0.574}}
& \cellcolor{blue!10}\textcolor{Maroon}{\textbf{0.647}} \\

& MAS\textcolor{teal}{$_{\text{ins}}\uparrow$}
& 0.186
& 0.268
& 0.316
& 0.324
& 0.327
& 0.304
& \underline{\textcolor{NavyBlue}{0.377}}
& \cellcolor{blue!10}\textcolor{Maroon}{\textbf{0.497}} \\

& RISE\textcolor{BurntOrange}{$_{\text{del}}\downarrow$}
& 0.345
& 0.307
& 0.274
& 0.269
& 0.331
& 0.274
& \underline{\textcolor{NavyBlue}{0.241}}
& \cellcolor{blue!10}\textcolor{Maroon}{\textbf{0.178}} \\

& MAS\textcolor{BurntOrange}{$_{\text{del}}\downarrow$}
& 0.521
& 0.494
& 0.435
& 0.427
& 0.464
& 0.422
& \underline{\textcolor{NavyBlue}{0.379}}
& \cellcolor{blue!10}\textcolor{Maroon}{\textbf{0.246}} \\
\cmidrule(lr){1-10}

\multirow{4}{*}{MMMU}
& RISE\textcolor{teal}{$_{\text{ins}}\uparrow$}
& 0.368
& 0.589
& 0.583
& 0.608
& 0.604
& 0.615
& \underline{\textcolor{NavyBlue}{0.629}}
& \cellcolor{blue!10}\textcolor{Maroon}{\textbf{0.660}} \\

& MAS\textcolor{teal}{$_{\text{ins}}\uparrow$}
& 0.170
& 0.394
& 0.393
& 0.429
& \underline{\textcolor{NavyBlue}{0.450}}
& 0.421
& 0.436
& \cellcolor{blue!10}\textcolor{Maroon}{\textbf{0.503}} \\

& RISE\textcolor{BurntOrange}{$_{\text{del}}\downarrow$}
& 0.315
& 0.216
& 0.204
& 0.195
& 0.241
& 0.191
& \underline{\textcolor{NavyBlue}{0.178}}
& \cellcolor{blue!10}\textcolor{Maroon}{\textbf{0.159}} \\

& MAS\textcolor{BurntOrange}{$_{\text{del}}\downarrow$}
& 0.468
& 0.356
& 0.343
& 0.323
& 0.334
& 0.294
& \underline{\textcolor{NavyBlue}{0.292}}
& \cellcolor{blue!10}\textcolor{Maroon}{\textbf{0.215}} \\
\cmidrule(lr){1-10}

\multirow{4}{*}{MMMU-Pro}
& RISE\textcolor{teal}{$_{\text{ins}}\uparrow$}
& 0.371
& 0.558
& 0.532
& 0.539
& 0.554
& 0.559
& \underline{\textcolor{NavyBlue}{0.584}}
& \cellcolor{blue!10}\textcolor{Maroon}{\textbf{0.619}} \\

& MAS\textcolor{teal}{$_{\text{ins}}\uparrow$}
& 0.166
& 0.348
& 0.338
& 0.351
& 0.379
& 0.349
& \underline{\textcolor{NavyBlue}{0.384}}
& \cellcolor{blue!10}\textcolor{Maroon}{\textbf{0.446}} \\

& RISE\textcolor{BurntOrange}{$_{\text{del}}\downarrow$}
& 0.322
& 0.241
& 0.242
& 0.242
& 0.269
& 0.229
& \underline{\textcolor{NavyBlue}{0.212}}
& \cellcolor{blue!10}\textcolor{Maroon}{\textbf{0.181}} \\

& MAS\textcolor{BurntOrange}{$_{\text{del}}\downarrow$}
& 0.477
& 0.401
& 0.403
& 0.402
& 0.377
& 0.355
& \underline{\textcolor{NavyBlue}{0.350}}
& \cellcolor{blue!10}\textcolor{Maroon}{\textbf{0.246}} \\
\cmidrule(lr){1-10}

\multirow{4}{*}{MathVerse}
& RISE\textcolor{teal}{$_{\text{ins}}\uparrow$}
& 0.383
& 0.495
& 0.539
& 0.540
& 0.513
& 0.547
& \underline{\textcolor{NavyBlue}{0.568}}
& \cellcolor{blue!10}\textcolor{Maroon}{\textbf{0.606}} \\

& MAS\textcolor{teal}{$_{\text{ins}}\uparrow$}
& 0.193
& 0.297
& 0.337
& 0.345
& 0.345
& 0.331
& \underline{\textcolor{NavyBlue}{0.350}}
& \cellcolor{blue!10}\textcolor{Maroon}{\textbf{0.434}} \\

& RISE\textcolor{BurntOrange}{$_{\text{del}}\downarrow$}
& 0.330
& 0.314
& 0.279
& 0.283
& 0.342
& 0.274
& \underline{\textcolor{NavyBlue}{0.241}}
& \cellcolor{blue!10}\textcolor{Maroon}{\textbf{0.187}} \\

& MAS\textcolor{BurntOrange}{$_{\text{del}}\downarrow$}
& 0.489
& 0.508
& 0.438
& 0.447
& 0.485
& 0.416
& \underline{\textcolor{NavyBlue}{0.375}}
& \cellcolor{blue!10}\textcolor{Maroon}{\textbf{0.254}} \\
\cmidrule(lr){1-10}

\multirow{4}{*}{VisualPuzzles}
& RISE\textcolor{teal}{$_{\text{ins}}\uparrow$}
& 0.376
& 0.488
& 0.477
& 0.504
& 0.523
& 0.529
& \underline{\textcolor{NavyBlue}{0.530}}
& \cellcolor{blue!10}\textcolor{Maroon}{\textbf{0.571}} \\

& MAS\textcolor{teal}{$_{\text{ins}}\uparrow$}
& 0.187
& 0.264
& 0.260
& 0.303
& \underline{\textcolor{NavyBlue}{0.323}}
& 0.321
& 0.319
& \cellcolor{blue!10}\textcolor{Maroon}{\textbf{0.366}} \\

& RISE\textcolor{BurntOrange}{$_{\text{del}}\downarrow$}
& 0.348
& 0.288
& 0.291
& 0.283
& 0.283
& \underline{\textcolor{NavyBlue}{0.255}}
& 0.268
& \cellcolor{blue!10}\textcolor{Maroon}{\textbf{0.207}} \\

& MAS\textcolor{BurntOrange}{$_{\text{del}}\downarrow$}
& 0.494
& 0.487
& 0.479
& 0.459
& 0.430
& \underline{\textcolor{NavyBlue}{0.400}}
& 0.434
& \cellcolor{blue!10}\textcolor{Maroon}{\textbf{0.292}} \\

\bottomrule
\end{tabular}
}
\label{tab:appendix_joint_faithfulness}
\end{table}

\subsection{Hyperparameter Sensitivity Study}
\label{appendix:hyperparameter_ifr}

\begin{figure}[h]
    \centering
    \includegraphics[width=\linewidth]{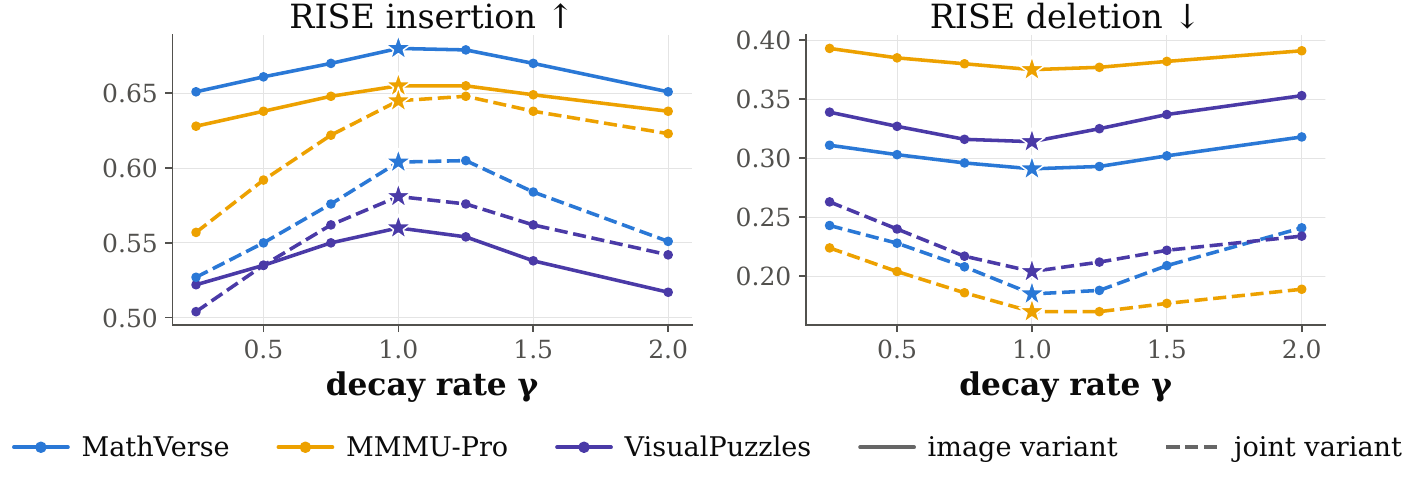}
    \caption{Sensitivity analysis for \textsc{VTrace} to the decay rate~$\gamma$.}
    \label{fig:ablation_hparam_VTrace}
\end{figure}

\begin{figure*}[h]
    \centering
    \includegraphics[width=\linewidth]{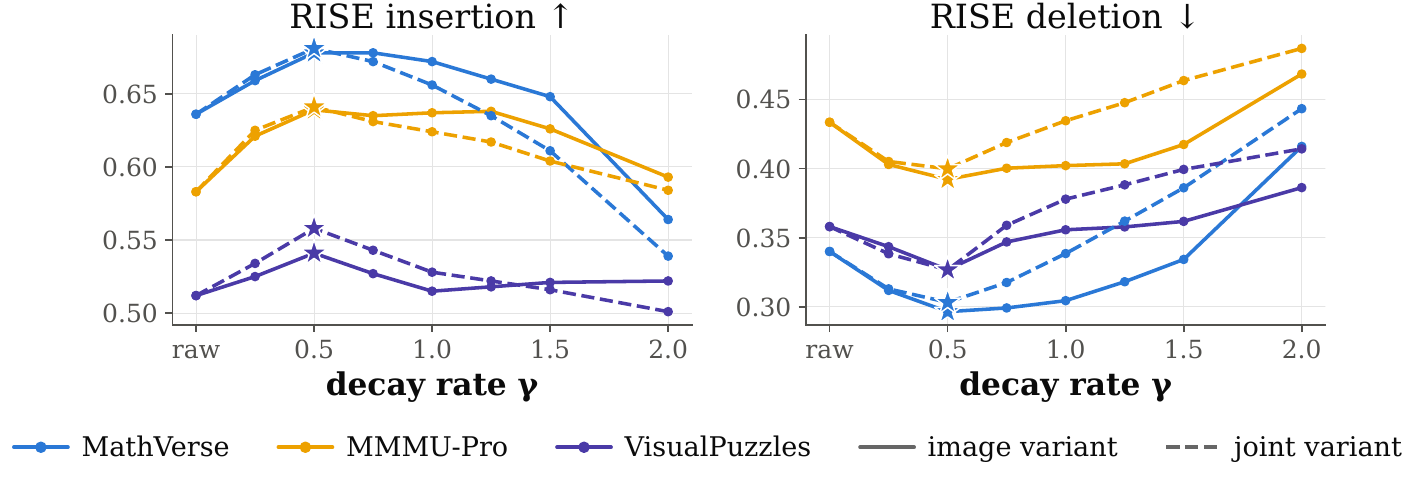}

    \caption{Sensitivity analysis for IFR to the decay rate~$\gamma$.}
    \label{fig:ablation_hparam_ifr} 
\end{figure*}

We conduct sensitivity study for decay rate $\gamma$. As shown in~\Cref{fig:ablation_hparam_VTrace}, $\gamma=1.0$ achieves the strongest performance across benchmarks, indicating that propagating attribution through intermediate tokens helps recover source evidence missed by direct attribution.
To examine whether the proposed multihop aggregation generalizes beyond \textsc{VTrace}, we additionally apply it to the baseline method IFR~\cite{IFR}. As shown in~\Cref{fig:ablation_hparam_ifr}, IFR exhibits a similar sensitivity pattern, with the strongest overall performance around $\gamma=0.5$. The results indicate that the proposed multi-hop aggregation can also enhance existing attribution methods.

\subsection{Full Results on Generalizability Across LVLM Families and Sizes}
\label{appendix:generalizability_experiments}
We report the complete attribution results on InternVL3.5-2B, Qwen3-VL-4B, and InternVL3.5-8B in \Cref{tab:appendix_generalizability_qwen4b,tab:appendix_generalizability_internvl2b,tab:appendix_generalizability_internvl8b}. The evaluation includes both RISE and MAS scores under the Image and Joint variant. Across different LVLM families and model scales, \textsc{VTRACE} consistently outperforms the baseline attribution methods and demonstrates strong generalizability.

\begin{table}[t]
\centering
\caption{
Attribution faithfulness on \textbf{Qwen3-VL-4B}.
}
\label{tab:appendix_generalizability_qwen4b}
\resizebox{\linewidth}{!}{
\renewcommand{\arraystretch}{1.15}
\begin{tabular}{llcccccccc}
\toprule
Dataset & Metric
& \textbf{ReAGent}
& \textbf{HETA}
& \textbf{FlowTracer}
& \textbf{IFR}
& \textbf{Attn Rollout}
& \textbf{AttnLRP}
& \textbf{FlashTrace}
& \cellcolor{blue!10}\textbf{\textsc{VTrace}}\\
\midrule
\multicolumn{10}{l}{\textit{Image variant}}\\
\midrule

\multirow{4}{*}{MMStar}
& RISE\textcolor{teal}{$_{\text{ins}}\uparrow$}
& 0.531
& 0.531
& 0.531
& 0.525
& 0.555
& \underline{\textcolor{NavyBlue}{0.575}}
& 0.539
& \cellcolor{blue!10}\textcolor{Maroon}{\textbf{0.614}} \\
& MAS\textcolor{teal}{$_{\text{ins}}\uparrow$}
& 0.382
& 0.364
& 0.368
& 0.376
& \underline{\textcolor{NavyBlue}{0.421}}
& 0.419
& 0.399
& \cellcolor{blue!10}\textcolor{Maroon}{\textbf{0.482}} \\
& RISE\textcolor{BurntOrange}{$_{\text{del}}\downarrow$}
& 0.480
& 0.453
& 0.454
& 0.462
& 0.452
& \underline{\textcolor{NavyBlue}{0.418}}
& 0.447
& \cellcolor{blue!10}\textcolor{Maroon}{\textbf{0.386}} \\
& MAS\textcolor{BurntOrange}{$_{\text{del}}\downarrow$}
& 0.631
& 0.620
& 0.614
& 0.609
& 0.586
& \underline{\textcolor{NavyBlue}{0.577}}
& 0.587
& \cellcolor{blue!10}\textcolor{Maroon}{\textbf{0.531}} \\
\cmidrule(lr){1-10}

\multirow{4}{*}{MathVerse}
& RISE\textcolor{teal}{$_{\text{ins}}\uparrow$}
& 0.490
& 0.600
& 0.612
& 0.611
& \underline{\textcolor{NavyBlue}{0.618}}
& 0.596
& 0.617
& \cellcolor{blue!10}\textcolor{Maroon}{\textbf{0.673}} \\
& MAS\textcolor{teal}{$_{\text{ins}}\uparrow$}
& 0.321
& 0.434
& 0.453
& 0.474
& \underline{\textcolor{NavyBlue}{0.498}}
& 0.438
& 0.483
& \cellcolor{blue!10}\textcolor{Maroon}{\textbf{0.553}} \\
& RISE\textcolor{BurntOrange}{$_{\text{del}}\downarrow$}
& 0.411
& 0.294
& 0.288
& 0.285
& 0.313
& 0.290
& \underline{\textcolor{NavyBlue}{0.277}}
& \cellcolor{blue!10}\textcolor{Maroon}{\textbf{0.230}} \\
& MAS\textcolor{BurntOrange}{$_{\text{del}}\downarrow$}
& 0.570
& 0.425
& 0.415
& 0.398
& 0.417
& 0.422
& \underline{\textcolor{NavyBlue}{0.388}}
& \cellcolor{blue!10}\textcolor{Maroon}{\textbf{0.345}} \\
\cmidrule(lr){1-10}

\multirow{4}{*}{VisualPuzzles}
& RISE\textcolor{teal}{$_{\text{ins}}\uparrow$}
& 0.470
& 0.463
& 0.441
& 0.445
& 0.514
& \underline{\textcolor{NavyBlue}{0.521}}
& 0.457
& \cellcolor{blue!10}\textcolor{Maroon}{\textbf{0.551}} \\
& MAS\textcolor{teal}{$_{\text{ins}}\uparrow$}
& 0.319
& 0.267
& 0.247
& 0.256
& \underline{\textcolor{NavyBlue}{0.350}}
& 0.335
& 0.270
& \cellcolor{blue!10}\textcolor{Maroon}{\textbf{0.379}} \\
& RISE\textcolor{BurntOrange}{$_{\text{del}}\downarrow$}
& 0.425
& 0.381
& 0.406
& 0.405
& 0.379
& \underline{\textcolor{NavyBlue}{0.345}}
& 0.398
& \cellcolor{blue!10}\textcolor{Maroon}{\textbf{0.321}} \\
& MAS\textcolor{BurntOrange}{$_{\text{del}}\downarrow$}
& 0.556
& 0.599
& 0.624
& 0.616
& 0.528
& \underline{\textcolor{NavyBlue}{0.519}}
& 0.602
& \cellcolor{blue!10}\textcolor{Maroon}{\textbf{0.470}} \\
\midrule
\multicolumn{10}{l}{\textit{Joint variant}}\\
\midrule

\multirow{4}{*}{MMStar}
& RISE\textcolor{teal}{$_{\text{ins}}\uparrow$}
& 0.335
& 0.510
& 0.503
& 0.503
& 0.491
& 0.516
& \underline{\textcolor{NavyBlue}{0.526}}
& \cellcolor{blue!10}\textcolor{Maroon}{\textbf{0.548}} \\
& MAS\textcolor{teal}{$_{\text{ins}}\uparrow$}
& 0.149
& 0.310
& 0.300
& 0.301
& \underline{\textcolor{NavyBlue}{0.326}}
& 0.314
& 0.317
& \cellcolor{blue!10}\textcolor{Maroon}{\textbf{0.369}} \\
& RISE\textcolor{BurntOrange}{$_{\text{del}}\downarrow$}
& 0.289
& 0.192
& 0.181
& 0.183
& 0.240
& 0.169
& \underline{\textcolor{NavyBlue}{0.166}}
& \cellcolor{blue!10}\textcolor{Maroon}{\textbf{0.157}} \\
& MAS\textcolor{BurntOrange}{$_{\text{del}}\downarrow$}
& 0.417
& 0.311
& 0.296
& 0.304
& 0.329
& \underline{\textcolor{NavyBlue}{0.249}}
& 0.274
& \cellcolor{blue!10}\textcolor{Maroon}{\textbf{0.214}} \\
\cmidrule(lr){1-10}

\multirow{4}{*}{MathVerse}
& RISE\textcolor{teal}{$_{\text{ins}}\uparrow$}
& 0.323
& 0.533
& 0.550
& 0.554
& 0.480
& 0.534
& \underline{\textcolor{NavyBlue}{0.559}}
& \cellcolor{blue!10}\textcolor{Maroon}{\textbf{0.579}} \\
& MAS\textcolor{teal}{$_{\text{ins}}\uparrow$}
& 0.142
& 0.345
& 0.344
& \underline{\textcolor{NavyBlue}{0.356}}
& 0.326
& 0.328
& 0.333
& \cellcolor{blue!10}\textcolor{Maroon}{\textbf{0.398}} \\
& RISE\textcolor{BurntOrange}{$_{\text{del}}\downarrow$}
& 0.248
& 0.195
& 0.171
& 0.173
& 0.305
& 0.174
& \underline{\textcolor{NavyBlue}{0.148}}
& \cellcolor{blue!10}\textcolor{Maroon}{\textbf{0.116}} \\
& MAS\textcolor{BurntOrange}{$_{\text{del}}\downarrow$}
& 0.403
& 0.310
& 0.260
& 0.266
& 0.414
& 0.245
& \underline{\textcolor{NavyBlue}{0.229}}
& \cellcolor{blue!10}\textcolor{Maroon}{\textbf{0.198}} \\
\cmidrule(lr){1-10}

\multirow{4}{*}{VisualPuzzles}
& RISE\textcolor{teal}{$_{\text{ins}}\uparrow$}
& 0.355
& 0.505
& 0.464
& 0.481
& 0.514
& \underline{\textcolor{NavyBlue}{0.520}}
& 0.489
& \cellcolor{blue!10}\textcolor{Maroon}{\textbf{0.546}} \\
& MAS\textcolor{teal}{$_{\text{ins}}\uparrow$}
& 0.174
& 0.281
& 0.243
& 0.260
& \underline{\textcolor{NavyBlue}{0.328}}
& 0.313
& 0.259
& \cellcolor{blue!10}\textcolor{Maroon}{\textbf{0.344}} \\
& RISE\textcolor{BurntOrange}{$_{\text{del}}\downarrow$}
& 0.320
& 0.236
& 0.262
& 0.261
& 0.243
& \underline{\textcolor{NavyBlue}{0.229}}
& 0.234
& \cellcolor{blue!10}\textcolor{Maroon}{\textbf{0.188}} \\
& MAS\textcolor{BurntOrange}{$_{\text{del}}\downarrow$}
& 0.446
& 0.405
& 0.446
& 0.443
& \underline{\textcolor{NavyBlue}{0.360}}
& 0.361
& 0.395
& \cellcolor{blue!10}\textcolor{Maroon}{\textbf{0.265}} \\

\bottomrule
\end{tabular}
}
\end{table}

\begin{table}[H]
\centering
\caption{
Attribution faithfulness on \textbf{InternVL3.5-2B}.
}
\label{tab:appendix_generalizability_internvl2b}
\resizebox{\linewidth}{!}{
\renewcommand{\arraystretch}{1.15}
\begin{tabular}{llcccccccc}
\toprule
Dataset & Metric
& \textbf{ReAGent}
& \textbf{HETA}
& \textbf{FlowTracer}
& \textbf{IFR}
& \textbf{Attn Rollout}
& \textbf{AttnLRP}
& \textbf{FlashTrace}
& \cellcolor{blue!10}\textbf{\textsc{VTrace}}\\
\midrule
\multicolumn{10}{l}{\textit{Image variant}}\\
\midrule

\multirow{4}{*}{MMStar}
& RISE\textcolor{teal}{$_{\text{ins}}\uparrow$}
& 0.523
& 0.639
& 0.592
& 0.614
& \underline{\textcolor{NavyBlue}{0.659}}
& 0.618
& 0.635
& \cellcolor{blue!10}\textcolor{Maroon}{\textbf{0.694}} \\
& MAS\textcolor{teal}{$_{\text{ins}}\uparrow$}
& 0.348
& 0.492
& 0.459
& 0.486
& \underline{\textcolor{NavyBlue}{0.554}}
& 0.477
& 0.520
& \cellcolor{blue!10}\textcolor{Maroon}{\textbf{0.588}} \\
& RISE\textcolor{BurntOrange}{$_{\text{del}}\downarrow$}
& 0.478
& 0.381
& 0.411
& 0.388
& 0.387
& 0.384
& \underline{\textcolor{NavyBlue}{0.372}}
& \cellcolor{blue!10}\textcolor{Maroon}{\textbf{0.327}} \\
& MAS\textcolor{BurntOrange}{$_{\text{del}}\downarrow$}
& 0.650
& 0.533
& 0.547
& 0.524
& 0.510
& 0.532
& \underline{\textcolor{NavyBlue}{0.499}}
& \cellcolor{blue!10}\textcolor{Maroon}{\textbf{0.461}} \\
\cmidrule(lr){1-10}

\multirow{4}{*}{MathVerse}
& RISE\textcolor{teal}{$_{\text{ins}}\uparrow$}
& 0.405
& 0.653
& 0.623
& 0.659
& \underline{\textcolor{NavyBlue}{0.680}}
& 0.616
& 0.672
& \cellcolor{blue!10}\textcolor{Maroon}{\textbf{0.709}} \\
& MAS\textcolor{teal}{$_{\text{ins}}\uparrow$}
& 0.245
& 0.512
& 0.487
& 0.540
& \underline{\textcolor{NavyBlue}{0.581}}
& 0.478
& 0.561
& \cellcolor{blue!10}\textcolor{Maroon}{\textbf{0.601}} \\
& RISE\textcolor{BurntOrange}{$_{\text{del}}\downarrow$}
& 0.310
& 0.181
& 0.193
& 0.173
& 0.198
& 0.173
& \underline{\textcolor{NavyBlue}{0.165}}
& \cellcolor{blue!10}\textcolor{Maroon}{\textbf{0.139}} \\
& MAS\textcolor{BurntOrange}{$_{\text{del}}\downarrow$}
& 0.443
& 0.265
& 0.293
& 0.266
& 0.306
& \underline{\textcolor{NavyBlue}{0.261}}
& 0.266
& \cellcolor{blue!10}\textcolor{Maroon}{\textbf{0.239}} \\
\cmidrule(lr){1-10}

\multirow{4}{*}{VisualPuzzles}
& RISE\textcolor{teal}{$_{\text{ins}}\uparrow$}
& 0.464
& 0.562
& 0.511
& 0.512
& \underline{\textcolor{NavyBlue}{0.599}}
& 0.544
& 0.534
& \cellcolor{blue!10}\textcolor{Maroon}{\textbf{0.641}} \\
& MAS\textcolor{teal}{$_{\text{ins}}\uparrow$}
& 0.260
& 0.368
& 0.328
& 0.322
& \underline{\textcolor{NavyBlue}{0.449}}
& 0.356
& 0.355
& \cellcolor{blue!10}\textcolor{Maroon}{\textbf{0.492}} \\
& RISE\textcolor{BurntOrange}{$_{\text{del}}\downarrow$}
& 0.401
& 0.322
& 0.370
& 0.360
& 0.327
& \underline{\textcolor{NavyBlue}{0.321}}
& 0.343
& \cellcolor{blue!10}\textcolor{Maroon}{\textbf{0.256}} \\
& MAS\textcolor{BurntOrange}{$_{\text{del}}\downarrow$}
& 0.588
& 0.488
& 0.542
& 0.547
& \underline{\textcolor{NavyBlue}{0.450}}
& 0.478
& 0.517
& \cellcolor{blue!10}\textcolor{Maroon}{\textbf{0.372}} \\
\midrule
\multicolumn{10}{l}{\textit{Joint variant}}\\
\midrule

\multirow{4}{*}{MMStar}
& RISE\textcolor{teal}{$_{\text{ins}}\uparrow$}
& 0.469
& \textcolor{Maroon}{\textbf{0.714}}
& 0.677
& 0.682
& 0.686
& 0.694
& 0.706
& \cellcolor{blue!10}\underline{\textcolor{NavyBlue}{0.707}} \\
& MAS\textcolor{teal}{$_{\text{ins}}\uparrow$}
& 0.254
& 0.548
& 0.505
& 0.516
& \underline{\textcolor{NavyBlue}{0.580}}
& 0.518
& 0.541
& \cellcolor{blue!10}\textcolor{Maroon}{\textbf{0.585}} \\
& RISE\textcolor{BurntOrange}{$_{\text{del}}\downarrow$}
& 0.372
& \underline{\textcolor{NavyBlue}{0.191}}
& 0.204
& 0.199
& 0.241
& 0.192
& 0.194
& \cellcolor{blue!10}\textcolor{Maroon}{\textbf{0.189}} \\
& MAS\textcolor{BurntOrange}{$_{\text{del}}\downarrow$}
& 0.587
& 0.321
& 0.339
& 0.333
& 0.316
& \underline{\textcolor{NavyBlue}{0.308}}
& 0.327
& \cellcolor{blue!10}\textcolor{Maroon}{\textbf{0.291}} \\
\cmidrule(lr){1-10}

\multirow{4}{*}{MathVerse}
& RISE\textcolor{teal}{$_{\text{ins}}\uparrow$}
& 0.481
& \underline{\textcolor{NavyBlue}{0.765}}
& 0.741
& 0.749
& 0.728
& 0.763
& 0.759
& \cellcolor{blue!10}\textcolor{Maroon}{\textbf{0.799}} \\
& MAS\textcolor{teal}{$_{\text{ins}}\uparrow$}
& 0.252
& 0.626
& 0.603
& 0.618
& \underline{\textcolor{NavyBlue}{0.642}}
& 0.627
& 0.610
& \cellcolor{blue!10}\textcolor{Maroon}{\textbf{0.722}} \\
& RISE\textcolor{BurntOrange}{$_{\text{del}}\downarrow$}
& 0.358
& 0.200
& 0.199
& 0.195
& 0.250
& 0.175
& \underline{\textcolor{NavyBlue}{0.172}}
& \cellcolor{blue!10}\textcolor{Maroon}{\textbf{0.153}} \\
& MAS\textcolor{BurntOrange}{$_{\text{del}}\downarrow$}
& 0.581
& 0.311
& 0.305
& 0.298
& 0.321
& \underline{\textcolor{NavyBlue}{0.263}}
& 0.271
& \cellcolor{blue!10}\textcolor{Maroon}{\textbf{0.209}} \\
\cmidrule(lr){1-10}

\multirow{4}{*}{VisualPuzzles}
& RISE\textcolor{teal}{$_{\text{ins}}\uparrow$}
& 0.433
& \underline{\textcolor{NavyBlue}{0.668}}
& 0.605
& 0.603
& \underline{\textcolor{NavyBlue}{0.668}}
& 0.613
& 0.631
& \cellcolor{blue!10}\textcolor{Maroon}{\textbf{0.690}} \\
& MAS\textcolor{teal}{$_{\text{ins}}\uparrow$}
& 0.217
& 0.477
& 0.409
& 0.413
& \underline{\textcolor{NavyBlue}{0.550}}
& 0.417
& 0.434
& \cellcolor{blue!10}\textcolor{Maroon}{\textbf{0.557}} \\
& RISE\textcolor{BurntOrange}{$_{\text{del}}\downarrow$}
& 0.350
& 0.221
& 0.254
& 0.249
& 0.242
& \underline{\textcolor{NavyBlue}{0.219}}
& 0.237
& \cellcolor{blue!10}\textcolor{Maroon}{\textbf{0.193}} \\
& MAS\textcolor{BurntOrange}{$_{\text{del}}\downarrow$}
& 0.556
& 0.351
& 0.409
& 0.403
& \underline{\textcolor{NavyBlue}{0.310}}
& 0.332
& 0.385
& \cellcolor{blue!10}\textcolor{Maroon}{\textbf{0.256}} \\

\bottomrule
\end{tabular}
}
\end{table}

\begin{table}[t]
\centering
\caption{
Attribution faithfulness on \textbf{InternVL3.5-8B}.
}
\label{tab:appendix_generalizability_internvl8b}
\resizebox{\linewidth}{!}{
\renewcommand{\arraystretch}{1.15}
\begin{tabular}{llcccccccc}
\toprule
Dataset & Metric
& \textbf{ReAGent}
& \textbf{HETA}
& \textbf{FlowTracer}
& \textbf{IFR}
& \textbf{Attn Rollout}
& \textbf{AttnLRP}
& \textbf{FlashTrace}
& \cellcolor{blue!10}\textbf{\textsc{VTrace}}\\
\midrule
\multicolumn{10}{l}{\textit{Image variant}}\\
\midrule

\multirow{4}{*}{MMStar}
& RISE\textcolor{teal}{$_{\text{ins}}\uparrow$}
& 0.529
& 0.684
& 0.637
& 0.679
& 0.681
& 0.646
& \underline{\textcolor{NavyBlue}{0.694}}
& \cellcolor{blue!10}\textcolor{Maroon}{\textbf{0.726}} \\
& MAS\textcolor{teal}{$_{\text{ins}}\uparrow$}
& 0.364
& 0.570
& 0.506
& 0.572
& 0.580
& 0.517
& \underline{\textcolor{NavyBlue}{0.592}}
& \cellcolor{blue!10}\textcolor{Maroon}{\textbf{0.633}} \\
& RISE\textcolor{BurntOrange}{$_{\text{del}}\downarrow$}
& 0.479
& 0.344
& 0.365
& 0.338
& 0.370
& 0.367
& \underline{\textcolor{NavyBlue}{0.325}}
& \cellcolor{blue!10}\textcolor{Maroon}{\textbf{0.303}} \\
& MAS\textcolor{BurntOrange}{$_{\text{del}}\downarrow$}
& 0.643
& 0.475
& 0.505
& 0.462
& 0.491
& 0.509
& \underline{\textcolor{NavyBlue}{0.447}}
& \cellcolor{blue!10}\textcolor{Maroon}{\textbf{0.426}} \\
\cmidrule(lr){1-10}

\multirow{4}{*}{MathVerse}
& RISE\textcolor{teal}{$_{\text{ins}}\uparrow$}
& 0.421
& 0.718
& 0.695
& 0.732
& 0.733
& 0.662
& \underline{\textcolor{NavyBlue}{0.742}}
& \cellcolor{blue!10}\textcolor{Maroon}{\textbf{0.753}} \\
& MAS\textcolor{teal}{$_{\text{ins}}\uparrow$}
& 0.298
& 0.620
& 0.579
& 0.632
& 0.624
& 0.545
& \underline{\textcolor{NavyBlue}{0.648}}
& \cellcolor{blue!10}\textcolor{Maroon}{\textbf{0.661}} \\
& RISE\textcolor{BurntOrange}{$_{\text{del}}\downarrow$}
& 0.322
& 0.170
& 0.174
& 0.157
& 0.174
& 0.189
& \underline{\textcolor{NavyBlue}{0.154}}
& \cellcolor{blue!10}\textcolor{Maroon}{\textbf{0.144}} \\
& MAS\textcolor{BurntOrange}{$_{\text{del}}\downarrow$}
& 0.456
& 0.260
& {\textcolor{Maroon}{\textbf{0.248}}}
& 0.284
& 0.331
& 0.269
& 0.270
& \cellcolor{blue!10}\underline{\textcolor{NavyBlue}{0.251}} \\
\cmidrule(lr){1-10}

\multirow{4}{*}{VisualPuzzles}
& RISE\textcolor{teal}{$_{\text{ins}}\uparrow$}
& 0.415
& 0.565
& 0.494
& 0.543
& {\textcolor{Maroon}{\textbf{0.622}}}
& 0.540
& 0.550
& \cellcolor{blue!10}\underline{\textcolor{NavyBlue}{0.617}} \\
& MAS\textcolor{teal}{$_{\text{ins}}\uparrow$}
& 0.256
& 0.414
& 0.330
& 0.399
& {\textcolor{Maroon}{\textbf{0.508}}}
& 0.372
& 0.408
& \cellcolor{blue!10}\underline{\textcolor{NavyBlue}{0.487}} \\
& RISE\textcolor{BurntOrange}{$_{\text{del}}\downarrow$}
& 0.379
& 0.278
& 0.345
& 0.304
& \underline{\textcolor{NavyBlue}{0.270}}
& 0.278
& 0.296
& \cellcolor{blue!10}\textcolor{Maroon}{\textbf{0.239}} \\
& MAS\textcolor{BurntOrange}{$_{\text{del}}\downarrow$}
& 0.529
& 0.404
& 0.512
& 0.455
& \underline{\textcolor{NavyBlue}{0.378}}
& 0.404
& 0.438
& \cellcolor{blue!10}\textcolor{Maroon}{\textbf{0.353}} \\
\midrule
\multicolumn{10}{l}{\textit{Joint variant}}\\
\midrule

\multirow{4}{*}{MMStar}
& RISE\textcolor{teal}{$_{\text{ins}}\uparrow$}
& 0.400
& \underline{\textcolor{NavyBlue}{0.727}}
& 0.700
& 0.715
& 0.537
& 0.693
& 0.708
& \cellcolor{blue!10}\textcolor{Maroon}{\textbf{0.760}} \\
& MAS\textcolor{teal}{$_{\text{ins}}\uparrow$}
& 0.184
& \underline{\textcolor{NavyBlue}{0.583}}
& 0.548
& 0.578
& 0.361
& 0.532
& 0.543
& \cellcolor{blue!10}\textcolor{Maroon}{\textbf{0.657}} \\
& RISE\textcolor{BurntOrange}{$_{\text{del}}\downarrow$}
& 0.286
& 0.135
& 0.130
& 0.123
& 0.245
& 0.132
& {\textcolor{Maroon}{\textbf{0.101}}}
& \cellcolor{blue!10}\underline{\textcolor{NavyBlue}{0.103}} \\
& MAS\textcolor{BurntOrange}{$_{\text{del}}\downarrow$}
& 0.457
& 0.212
& 0.204
& 0.193
& 0.341
& 0.194
& \underline{\textcolor{NavyBlue}{0.157}}
& \cellcolor{blue!10}\textcolor{Maroon}{\textbf{0.149}} \\
\cmidrule(lr){1-10}

\multirow{4}{*}{MathVerse}
& RISE\textcolor{teal}{$_{\text{ins}}\uparrow$}
& 0.457
& 0.786
& 0.775
& 0.780
& 0.673
& 0.760
& \underline{\textcolor{NavyBlue}{0.794}}
& \cellcolor{blue!10}\textcolor{Maroon}{\textbf{0.854}} \\
& MAS\textcolor{teal}{$_{\text{ins}}\uparrow$}
& 0.272
& 0.675
& 0.675
& \underline{\textcolor{NavyBlue}{0.685}}
& 0.540
& 0.639
& 0.681
& \cellcolor{blue!10}\textcolor{Maroon}{\textbf{0.800}} \\
& RISE\textcolor{BurntOrange}{$_{\text{del}}\downarrow$}
& 0.354
& 0.190
& 0.181
& 0.173
& 0.331
& 0.182
& \underline{\textcolor{NavyBlue}{0.147}}
& \cellcolor{blue!10}\textcolor{Maroon}{\textbf{0.116}} \\
& MAS\textcolor{BurntOrange}{$_{\text{del}}\downarrow$}
& 0.513
& 0.278
& 0.266
& 0.254
& 0.456
& 0.266
& \underline{\textcolor{NavyBlue}{0.217}}
& \cellcolor{blue!10}\textcolor{Maroon}{\textbf{0.168}} \\
\cmidrule(lr){1-10}

\multirow{4}{*}{VisualPuzzles}
& RISE\textcolor{teal}{$_{\text{ins}}\uparrow$}
& 0.376
& \underline{\textcolor{NavyBlue}{0.648}}
& 0.578
& 0.603
& 0.586
& 0.588
& 0.576
& \cellcolor{blue!10}\textcolor{Maroon}{\textbf{0.666}} \\
& MAS\textcolor{teal}{$_{\text{ins}}\uparrow$}
& 0.181
& \underline{\textcolor{NavyBlue}{0.470}}
& 0.407
& 0.447
& 0.420
& 0.401
& 0.359
& \cellcolor{blue!10}\textcolor{Maroon}{\textbf{0.531}} \\
& RISE\textcolor{BurntOrange}{$_{\text{del}}\downarrow$}
& 0.304
& 0.189
& 0.223
& 0.208
& 0.207
& 0.184
& \underline{\textcolor{NavyBlue}{0.181}}
& \cellcolor{blue!10}\textcolor{Maroon}{\textbf{0.135}} \\
& MAS\textcolor{BurntOrange}{$_{\text{del}}\downarrow$}
& 0.452
& 0.295
& 0.362
& 0.337
& 0.280
& \underline{\textcolor{NavyBlue}{0.274}}
& 0.299
& \cellcolor{blue!10}\textcolor{Maroon}{\textbf{0.188}} \\

\bottomrule
\end{tabular}
}
\end{table}

\begin{table*}[t]
\centering
\caption{
Attribution faithfulness on correctly and incorrectly
answered samples for the \textbf{Joint variant}.
RISE and MAS are reported under insertion
(\textcolor{teal}{$_{\text{ins}}\uparrow$}, higher is better)
and deletion
(\textcolor{BurntOrange}{$_{\text{del}}\downarrow$}, lower is better).
Best results are highlighted in
\textcolor{Maroon}{\textbf{red bold}}, and second-best results are highlighted
in \textcolor{NavyBlue}{\underline{blue underlining}}.
}
\label{tab:correctness_joint_full}
\resizebox{\linewidth}{!}{
\renewcommand{\arraystretch}{1.12}
\begin{tabular}{lllcccccccc}
\toprule
Dataset & Split & Metric
& \textbf{ReAGent}
& \textbf{HETA}
& \textbf{FlowTracer}
& \textbf{IFR}
& \textbf{Attn Rollout}
& \textbf{AttnLRP}
& \textbf{FlashTrace}
& \cellcolor{blue!10}\textbf{\textsc{VTrace}} \\
\midrule

\multirow{8}{*}{\shortstack{MMStar}}
& \multirow{4}{*}{Correct}
& RISE\textcolor{teal}{$_{\text{ins}}\uparrow$}
& 0.371 & 0.482 & 0.504 & 0.504 & 0.497 & 0.521
& \underline{\textcolor{NavyBlue}{0.552}}
& \cellcolor{blue!10}\textcolor{Maroon}{\textbf{0.581}} \\

&
& MAS\textcolor{teal}{$_{\text{ins}}\uparrow$}
& 0.166 & 0.271 & 0.297 & 0.303 & 0.310 & 0.306
& \underline{\textcolor{NavyBlue}{0.347}}
& \cellcolor{blue!10}\textcolor{Maroon}{\textbf{0.400}} \\

&
& RISE\textcolor{BurntOrange}{$_{\text{del}}\downarrow$}
& 0.326 & 0.251 & 0.230 & 0.225 & 0.278 & 0.225
& \underline{\textcolor{NavyBlue}{0.207}}
& \cellcolor{blue!10}\textcolor{Maroon}{\textbf{0.181}} \\

&
& MAS\textcolor{BurntOrange}{$_{\text{del}}\downarrow$}
& 0.494 & 0.417 & 0.377 & 0.369 & 0.390 & 0.351
& \underline{\textcolor{NavyBlue}{0.336}}
& \cellcolor{blue!10}\textcolor{Maroon}{\textbf{0.248}} \\
\cmidrule(lr){2-11}

& \multirow{4}{*}{Incorrect}
& RISE\textcolor{teal}{$_{\text{ins}}\uparrow$}
& 0.371 & 0.511 & 0.514 & 0.514 & 0.513 & 0.551
& \underline{\textcolor{NavyBlue}{0.559}}
& \cellcolor{blue!10}\textcolor{Maroon}{\textbf{0.580}} \\

&
& MAS\textcolor{teal}{$_{\text{ins}}\uparrow$}
& 0.164 & 0.295 & 0.305 & 0.311 & 0.324 & 0.337
& \underline{\textcolor{NavyBlue}{0.352}}
& \cellcolor{blue!10}\textcolor{Maroon}{\textbf{0.389}} \\

&
& RISE\textcolor{BurntOrange}{$_{\text{del}}\downarrow$}
& 0.322 & 0.227 & 0.216 & 0.212 & 0.257 & 0.202
& \underline{\textcolor{NavyBlue}{0.195}}
& \cellcolor{blue!10}\textcolor{Maroon}{\textbf{0.177}} \\

&
& MAS\textcolor{BurntOrange}{$_{\text{del}}\downarrow$}
& 0.492 & 0.381 & 0.358 & 0.350 & 0.365
& \underline{\textcolor{NavyBlue}{0.315}}
& 0.320
& \cellcolor{blue!10}\textcolor{Maroon}{\textbf{0.241}} \\
\midrule

\multirow{8}{*}{\shortstack{MathVista}}
& \multirow{4}{*}{Correct}
& RISE\textcolor{teal}{$_{\text{ins}}\uparrow$}
& 0.376 & 0.460 & 0.510 & 0.507 & 0.492 & 0.501
& \underline{\textcolor{NavyBlue}{0.567}}
& \cellcolor{blue!10}\textcolor{Maroon}{\textbf{0.644}} \\

&
& MAS\textcolor{teal}{$_{\text{ins}}\uparrow$}
& 0.177 & 0.262 & 0.309 & 0.314 & 0.317 & 0.292
& \underline{\textcolor{NavyBlue}{0.370}}
& \cellcolor{blue!10}\textcolor{Maroon}{\textbf{0.495}} \\

&
& RISE\textcolor{BurntOrange}{$_{\text{del}}\downarrow$}
& 0.341 & 0.308 & 0.276 & 0.272 & 0.336 & 0.276
& \underline{\textcolor{NavyBlue}{0.242}}
& \cellcolor{blue!10}\textcolor{Maroon}{\textbf{0.174}} \\

&
& MAS\textcolor{BurntOrange}{$_{\text{del}}\downarrow$}
& 0.522 & 0.496 & 0.438 & 0.430 & 0.470 & 0.424
& \underline{\textcolor{NavyBlue}{0.381}}
& \cellcolor{blue!10}\textcolor{Maroon}{\textbf{0.241}} \\
\cmidrule(lr){2-11}

& \multirow{4}{*}{Incorrect}
& RISE\textcolor{teal}{$_{\text{ins}}\uparrow$}
& 0.410 & 0.501 & 0.551 & 0.560 & 0.542 & 0.565
& \underline{\textcolor{NavyBlue}{0.606}}
& \cellcolor{blue!10}\textcolor{Maroon}{\textbf{0.657}} \\

&
& MAS\textcolor{teal}{$_{\text{ins}}\uparrow$}
& 0.228 & 0.299 & 0.347 & 0.373 & 0.376 & 0.361
& \underline{\textcolor{NavyBlue}{0.410}}
& \cellcolor{blue!10}\textcolor{Maroon}{\textbf{0.502}} \\

&
& RISE\textcolor{BurntOrange}{$_{\text{del}}\downarrow$}
& 0.364 & 0.299 & 0.264 & 0.257 & 0.310 & 0.265
& \underline{\textcolor{NavyBlue}{0.234}}
& \cellcolor{blue!10}\textcolor{Maroon}{\textbf{0.194}} \\

&
& MAS\textcolor{BurntOrange}{$_{\text{del}}\downarrow$}
& 0.519 & 0.484 & 0.420 & 0.409 & 0.437 & 0.413
& \underline{\textcolor{NavyBlue}{0.370}}
& \cellcolor{blue!10}\textcolor{Maroon}{\textbf{0.269}} \\
\midrule

\multirow{8}{*}{\shortstack{MMMU}}
& \multirow{4}{*}{Correct}
& RISE\textcolor{teal}{$_{\text{ins}}\uparrow$}
& 0.368 & 0.583 & 0.576 & 0.599 & 0.593 & 0.606
& \underline{\textcolor{NavyBlue}{0.623}}
& \cellcolor{blue!10}\textcolor{Maroon}{\textbf{0.654}} \\

&
& MAS\textcolor{teal}{$_{\text{ins}}\uparrow$}
& 0.169 & 0.386 & 0.383 & 0.415
& \underline{\textcolor{NavyBlue}{0.436}}
& 0.409 & 0.427
& \cellcolor{blue!10}\textcolor{Maroon}{\textbf{0.494}} \\

&
& RISE\textcolor{BurntOrange}{$_{\text{del}}\downarrow$}
& 0.314 & 0.217 & 0.206 & 0.197 & 0.246 & 0.193
& \underline{\textcolor{NavyBlue}{0.179}}
& \cellcolor{blue!10}\textcolor{Maroon}{\textbf{0.159}} \\

&
& MAS\textcolor{BurntOrange}{$_{\text{del}}\downarrow$}
& 0.466 & 0.359 & 0.347 & 0.327 & 0.341 & 0.297
& \underline{\textcolor{NavyBlue}{0.296}}
& \cellcolor{blue!10}\textcolor{Maroon}{\textbf{0.216}} \\
\cmidrule(lr){2-11}

& \multirow{4}{*}{Incorrect}
& RISE\textcolor{teal}{$_{\text{ins}}\uparrow$}
& 0.370 & 0.613 & 0.607 & 0.637 & 0.644 & 0.645
& \underline{\textcolor{NavyBlue}{0.650}}
& \cellcolor{blue!10}\textcolor{Maroon}{\textbf{0.682}} \\

&
& MAS\textcolor{teal}{$_{\text{ins}}\uparrow$}
& 0.176 & 0.425 & 0.428 & 0.473
& \underline{\textcolor{NavyBlue}{0.500}}
& 0.464 & 0.468
& \cellcolor{blue!10}\textcolor{Maroon}{\textbf{0.532}} \\

&
& RISE\textcolor{BurntOrange}{$_{\text{del}}\downarrow$}
& 0.322 & 0.211 & 0.199 & 0.187 & 0.223 & 0.183
& \underline{\textcolor{NavyBlue}{0.171}}
& \cellcolor{blue!10}\textcolor{Maroon}{\textbf{0.156}} \\

&
& MAS\textcolor{BurntOrange}{$_{\text{del}}\downarrow$}
& 0.474 & 0.343 & 0.330 & 0.308 & 0.310 & 0.282
& \underline{\textcolor{NavyBlue}{0.279}}
& \cellcolor{blue!10}\textcolor{Maroon}{\textbf{0.212}} \\

\bottomrule
\end{tabular}
}
\end{table*}

\begin{table*}[t]
\centering
\caption{
Attribution faithfulness on correctly and incorrectly
answered samples for the \textbf{Image variant}.
RISE and MAS are reported under insertion
(\textcolor{teal}{$_{\text{ins}}\uparrow$}, higher is better)
and deletion
(\textcolor{BurntOrange}{$_{\text{del}}\downarrow$}, lower is better).
Best results are highlighted in
\textcolor{Maroon}{\textbf{red bold}}, and second-best results are highlighted
in \textcolor{NavyBlue}{\underline{blue underlining}}.
}
\label{tab:correctness_image_full}
\resizebox{\linewidth}{!}{
\renewcommand{\arraystretch}{1.12}
\begin{tabular}{lllcccccccc}
\toprule
Dataset & Split & Metric
& \textbf{ReAGent}
& \textbf{HETA}
& \textbf{FlowTracer}
& \textbf{IFR}
& \textbf{Attn Rollout}
& \textbf{AttnLRP}
& \textbf{FlashTrace}
& \cellcolor{blue!10}\textbf{\textsc{VTrace}} \\
\midrule

\multirow{8}{*}{\shortstack{MMStar}}
& \multirow{4}{*}{Correct}
& RISE\textcolor{teal}{$_{\text{ins}}\uparrow$}
& 0.499 & 0.493 & 0.534 & 0.542 & 0.541
& \underline{\textcolor{NavyBlue}{0.560}}
& 0.556
& \cellcolor{blue!10}\textcolor{Maroon}{\textbf{0.603}} \\

&
& MAS\textcolor{teal}{$_{\text{ins}}\uparrow$}
& 0.328 & 0.326 & 0.371 & 0.390 & 0.403 & 0.389
& \underline{\textcolor{NavyBlue}{0.414}}
& \cellcolor{blue!10}\textcolor{Maroon}{\textbf{0.466}} \\

&
& RISE\textcolor{BurntOrange}{$_{\text{del}}\downarrow$}
& 0.453 & 0.459 & 0.402 & 0.399 & 0.429
& \underline{\textcolor{NavyBlue}{0.380}}
& 0.386
& \cellcolor{blue!10}\textcolor{Maroon}{\textbf{0.344}} \\

&
& MAS\textcolor{BurntOrange}{$_{\text{del}}\downarrow$}
& 0.616 & 0.618 & 0.549 & 0.536 & 0.552 & 0.539
& \underline{\textcolor{NavyBlue}{0.514}}
& \cellcolor{blue!10}\textcolor{Maroon}{\textbf{0.478}} \\
\cmidrule(lr){2-11}

& \multirow{4}{*}{Incorrect}
& RISE\textcolor{teal}{$_{\text{ins}}\uparrow$}
& 0.525 & 0.509 & 0.526 & 0.531 & 0.531
& \underline{\textcolor{NavyBlue}{0.572}}
& 0.550
& \cellcolor{blue!10}\textcolor{Maroon}{\textbf{0.591}} \\

&
& MAS\textcolor{teal}{$_{\text{ins}}\uparrow$}
& 0.357 & 0.336 & 0.360 & 0.369 & 0.388
& \underline{\textcolor{NavyBlue}{0.399}}
& \underline{\textcolor{NavyBlue}{0.399}}
& \cellcolor{blue!10}\textcolor{Maroon}{\textbf{0.442}} \\

&
& RISE\textcolor{BurntOrange}{$_{\text{del}}\downarrow$}
& 0.471 & 0.477 & 0.440 & 0.437 & 0.469
& \underline{\textcolor{NavyBlue}{0.397}}
& 0.418
& \cellcolor{blue!10}\textcolor{Maroon}{\textbf{0.374}} \\

&
& MAS\textcolor{BurntOrange}{$_{\text{del}}\downarrow$}
& 0.632 & 0.646 & 0.603 & 0.595 & 0.606 & 0.568
& \underline{\textcolor{NavyBlue}{0.566}}
& \cellcolor{blue!10}\textcolor{Maroon}{\textbf{0.520}} \\
\midrule

\multirow{8}{*}{\shortstack{MathVista}}
& \multirow{4}{*}{Correct}
& RISE\textcolor{teal}{$_{\text{ins}}\uparrow$}
& 0.496 & 0.509 & 0.576 & 0.586 & 0.580
& \underline{\textcolor{NavyBlue}{0.597}}
& \underline{\textcolor{NavyBlue}{0.597}}
& \cellcolor{blue!10}\textcolor{Maroon}{\textbf{0.663}} \\

&
& MAS\textcolor{teal}{$_{\text{ins}}\uparrow$}
& 0.326 & 0.335 & 0.418 & 0.443 & 0.446 & 0.424
& \underline{\textcolor{NavyBlue}{0.459}}
& \cellcolor{blue!10}\textcolor{Maroon}{\textbf{0.537}} \\

&
& RISE\textcolor{BurntOrange}{$_{\text{del}}\downarrow$}
& 0.440 & 0.439 & 0.358 & 0.356 & 0.388 & 0.348
& \underline{\textcolor{NavyBlue}{0.345}}
& \cellcolor{blue!10}\textcolor{Maroon}{\textbf{0.297}} \\

&
& MAS\textcolor{BurntOrange}{$_{\text{del}}\downarrow$}
& 0.611 & 0.599 & 0.500 & 0.489 & 0.510 & 0.506
& \underline{\textcolor{NavyBlue}{0.470}}
& \cellcolor{blue!10}\textcolor{Maroon}{\textbf{0.429}} \\
\cmidrule(lr){2-11}

& \multirow{4}{*}{Incorrect}
& RISE\textcolor{teal}{$_{\text{ins}}\uparrow$}
& 0.523 & 0.539 & 0.590 & 0.614 & 0.603 & 0.611
& \underline{\textcolor{NavyBlue}{0.617}}
& \cellcolor{blue!10}\textcolor{Maroon}{\textbf{0.655}} \\

&
& MAS\textcolor{teal}{$_{\text{ins}}\uparrow$}
& 0.373 & 0.364 & 0.434 & 0.480 & 0.478 & 0.442
& \underline{\textcolor{NavyBlue}{0.485}}
& \cellcolor{blue!10}\textcolor{Maroon}{\textbf{0.523}} \\

&
& RISE\textcolor{BurntOrange}{$_{\text{del}}\downarrow$}
& 0.483 & 0.469 & 0.416 & 0.396 & 0.422 & 0.405
& \underline{\textcolor{NavyBlue}{0.394}}
& \cellcolor{blue!10}\textcolor{Maroon}{\textbf{0.373}} \\

&
& MAS\textcolor{BurntOrange}{$_{\text{del}}\downarrow$}
& 0.634 & 0.644 & 0.575 & 0.546 & 0.558 & 0.588
& \underline{\textcolor{NavyBlue}{0.538}}
& \cellcolor{blue!10}\textcolor{Maroon}{\textbf{0.525}} \\
\midrule

\multirow{8}{*}{\shortstack{MMMU}}
& \multirow{4}{*}{Correct}
& RISE\textcolor{teal}{$_{\text{ins}}\uparrow$}
& 0.533 & 0.562 & 0.591 & 0.606
& \underline{\textcolor{NavyBlue}{0.624}}
& 0.607 & 0.616
& \cellcolor{blue!10}\textcolor{Maroon}{\textbf{0.666}} \\

&
& MAS\textcolor{teal}{$_{\text{ins}}\uparrow$}
& 0.372 & 0.389 & 0.430 & 0.458
& \underline{\textcolor{NavyBlue}{0.496}}
& 0.440 & 0.479
& \cellcolor{blue!10}\textcolor{Maroon}{\textbf{0.540}} \\

&
& RISE\textcolor{BurntOrange}{$_{\text{del}}\downarrow$}
& 0.470 & 0.463 & 0.421 & 0.409 & 0.428 & 0.403
& \underline{\textcolor{NavyBlue}{0.400}}
& \cellcolor{blue!10}\textcolor{Maroon}{\textbf{0.359}} \\

&
& MAS\textcolor{BurntOrange}{$_{\text{del}}\downarrow$}
& 0.626 & 0.623 & 0.585 & 0.555 & 0.560 & 0.566
& \underline{\textcolor{NavyBlue}{0.539}}
& \cellcolor{blue!10}\textcolor{Maroon}{\textbf{0.497}} \\
\cmidrule(lr){2-11}

& \multirow{4}{*}{Incorrect}
& RISE\textcolor{teal}{$_{\text{ins}}\uparrow$}
& 0.575 & 0.567 & 0.608 & 0.625
& \underline{\textcolor{NavyBlue}{0.638}}
& 0.627 & 0.627
& \cellcolor{blue!10}\textcolor{Maroon}{\textbf{0.661}} \\

&
& MAS\textcolor{teal}{$_{\text{ins}}\uparrow$}
& 0.412 & 0.390 & 0.435 & 0.472
& \underline{\textcolor{NavyBlue}{0.513}}
& 0.456 & 0.481
& \cellcolor{blue!10}\textcolor{Maroon}{\textbf{0.523}} \\

&
& RISE\textcolor{BurntOrange}{$_{\text{del}}\downarrow$}
& 0.522 & 0.541 & 0.493 & 0.477 & 0.497
& \underline{\textcolor{NavyBlue}{0.470}}
& 0.472
& \cellcolor{blue!10}\textcolor{Maroon}{\textbf{0.448}} \\

&
& MAS\textcolor{BurntOrange}{$_{\text{del}}\downarrow$}
& 0.683 & 0.717 & 0.676 & 0.647 & 0.647 & 0.637
& \underline{\textcolor{NavyBlue}{0.634}}
& \cellcolor{blue!10}\textcolor{Maroon}{\textbf{0.606}} \\

\bottomrule
\end{tabular}
}
\end{table*}

\subsection{Full Results on Robustness to Prediction Correctness}
\label{appendix:correctness_full_results}
We evaluate attribution faithfulness separately on correctly and incorrectly answered samples to examine whether the attribution quality depends on prediction correctness.
\Cref{tab:correctness_joint_full,tab:correctness_image_full} show the complete results for the Image and Joint Variants, where \textsc{VTRACE} achieves the best performance on both correct and incorrect predictions.
The results indicate that the attribution gains arise from attribution method rather than answer correctness.

\begin{figure*}[h]
    \centering
    \includegraphics[width=0.9\linewidth]{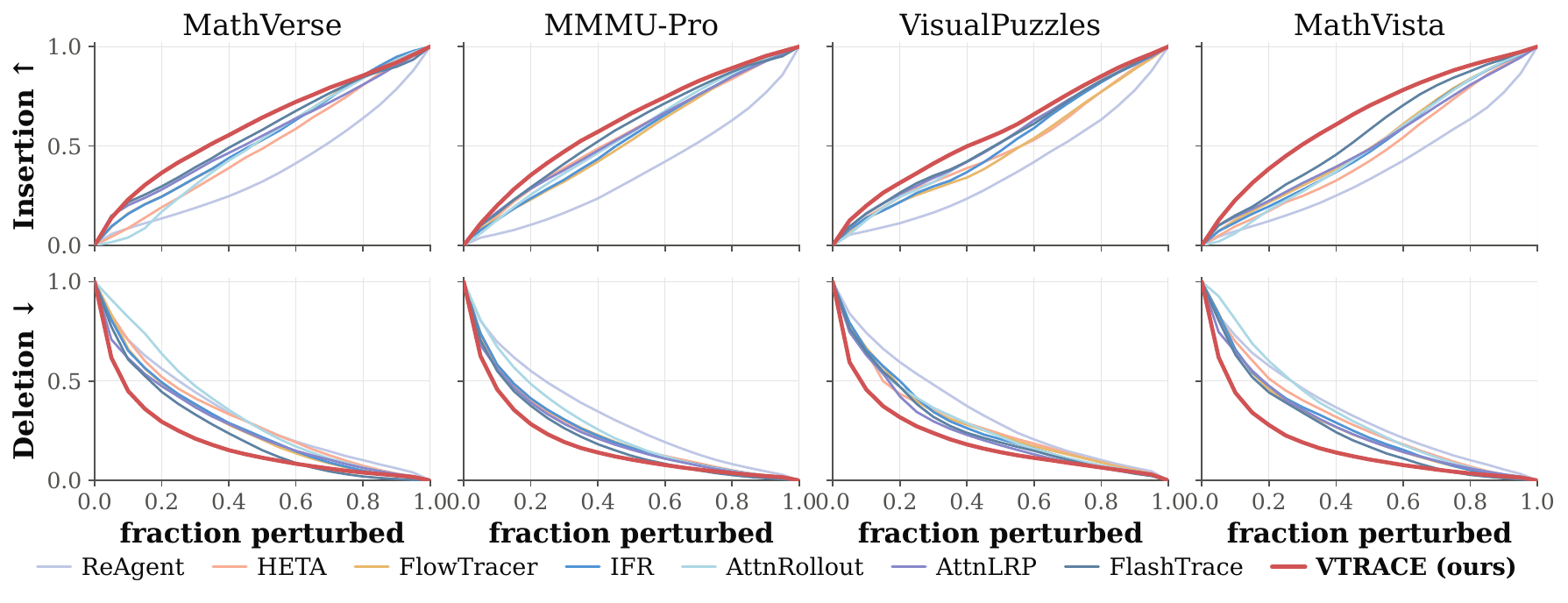}
    \caption{Fine-grained insertion and deletion perturbation curves across benchmarks.}
    \label{fig:fine_grained_curves} 
    \vspace{-0.3cm}
\end{figure*}

\subsection{Fine-Grained Perturbation Curves}
\label{appendix:fine_grained_plot}
\textsc{VTrace} produces more faithful rankings throughout the perturbation process, beyond what is captured by the AUC scores. Shown in Figure~\Cref{fig:fine_grained_curves}, \textsc{VTrace} achieves faster recovery under insertion and sharper degradation under deletion. The advantage is especially clear at early perturbation stages, where restoring only a small fraction of top-ranked tokens rapidly recovers the response, while removing them causes substantial degradation.

\subsection{Summing paths carries credit back to the inputs.}
\label{appendix:hops_ablation_performance}
Figure~\ref{fig:credit_distance} shows how far back a generated token's credit comes from. Evidently, a single hop keeps credit near the receiver: under $W$, 22\% of a token's credit falls on the token directly before it and 58\% on the 20 tokens before it (IFR: 17\% and 53\%). Once \textsc{VTrace} sums all direct and indirect paths ($R$), the token directly before keeps 5\%, and 75\% of the credit comes from sources more than 20 positions back. For a generated token, those sources are mostly the question, the image and the early reasoning. Figure~\ref{fig:image_credit_hops} tracks how much of the answer's credit lands on the image as longer paths are added. The direct edge gives the image 0.7\% and IFR gives it 1.0\%. The share rises with every hop, to 1.6\% with paths of up to two edges and 6.2\% with paths of up to eight, and $\R$ reaches 8.5\%. The direct edge therefore sees almost none of the image, and the image share is still rising after eight hops. Figure~\ref{fig:hops_performance} tests whether this extra credit is faithful. We rank the image patches and the question tokens separately by each operator's score, then remove them in that order (deletion) or add them back to a fully masked input (insertion) and measure the RISE. Notably, longer paths help in all four settings, and \textsc{VTrace}, which aggregates all paths performs the best, as expected.
\begin{figure*}[h]
    \centering
    \begin{subfigure}[b]{0.324\linewidth}
        \includegraphics[width=\linewidth]{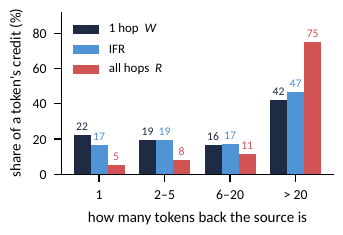}
        \caption{Source distance}
        \label{fig:credit_distance}
    \end{subfigure}\hfill
    \begin{subfigure}[b]{0.333\linewidth}
        \includegraphics[width=\linewidth]{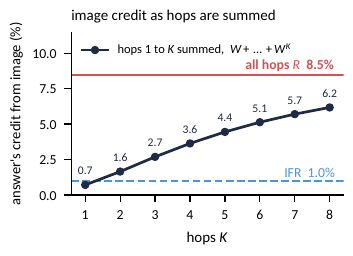}
        \caption{Image credit vs.\ hops}
        \label{fig:image_credit_hops}
    \end{subfigure}\hfill
    \begin{subfigure}[b]{0.313\linewidth}
        \includegraphics[width=\linewidth]{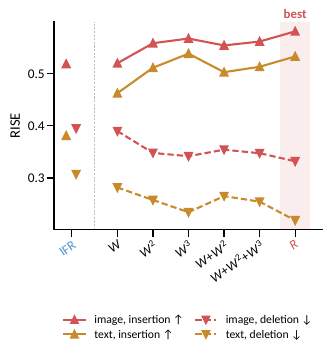}
        \caption{RISE by operator}
        \label{fig:hops_performance}
    \end{subfigure}
    \caption{(a) How far back a generated token's credit comes from. (b) The answer's image credit as hops are summed. (c) Performance comparison between diverse hops, \textsc{VTrace} ($R$) and IFR.}
    \label{fig:hops}
    \vspace{-0.3cm}
\end{figure*}

\section{Case Studies} \label{appendix:case_studies}
\textbf{Image token attribution.} Figure~\ref{fig:image_attribution_casestudy} compares the top 20\% image patches selected by each method against a perturbation-based reference. For each patch, we blur it and measure the resulting drop in the model's trace likelihood, where a larger drop indicates greater reliance on that region. This reference map is shown in the second column. The reported score on the bottom right measures how much of the reference importance is captured by the selected patches, with higher values indicating better alignment. \textbf{We demonstrate two better performing and two slightly behind cases of \textsc{VTrace}.} In the first two examples, \textsc{VTrace} focuses more strongly on the relevant people and scene regions, scoring 0.50 and 0.42 compared with 0.46 and 0.35 for the strongest baseline. The third example is tied with FlowTracer. The fourth shows a slightly less aligned case, while both \textsc{VTrace} and AttnLRP both captures the important balls, \textsc{VTrace} also assigns attribution to the player and table edge. This suggests that \textsc{VTrace} can miss important visual regions when they are not clearly reflected in the generated reasoning trace.

\begin{figure}[h!]
    \centering
    \includegraphics[width=1\linewidth]{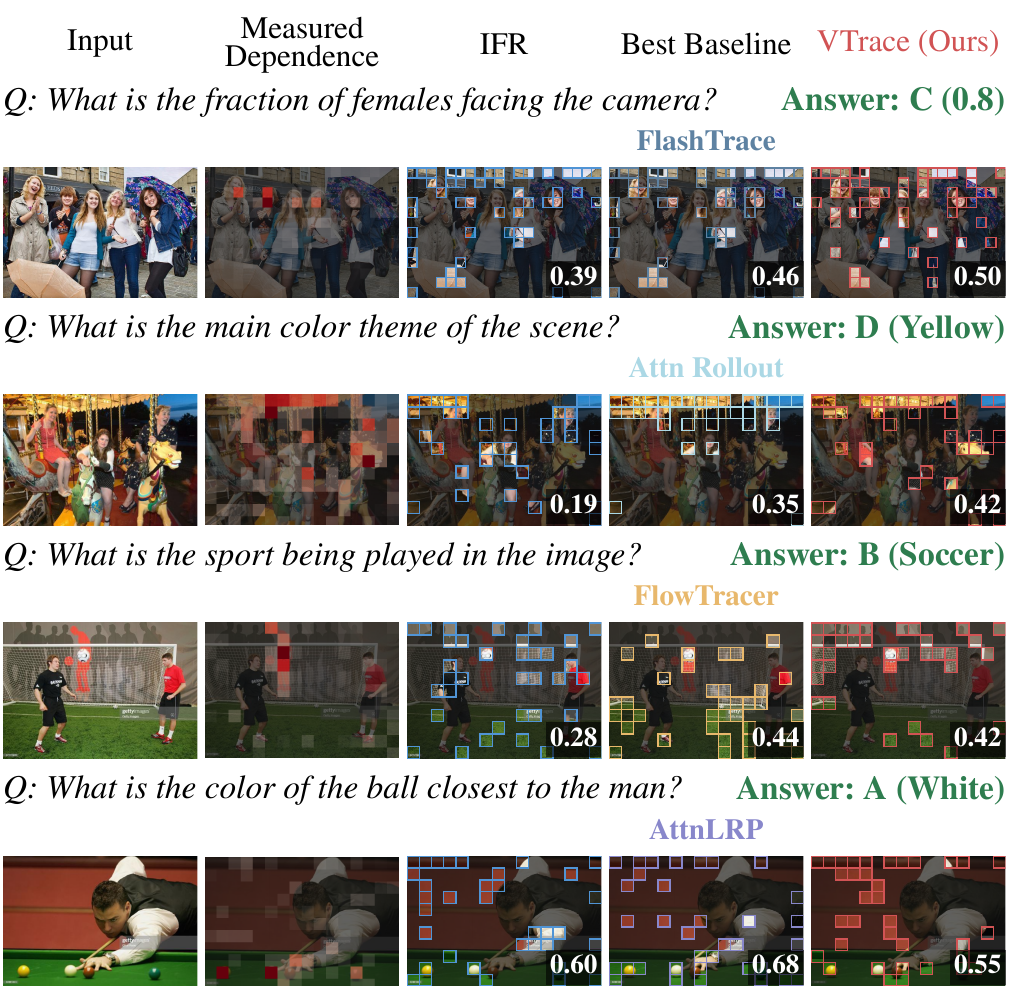}
    \caption{\textbf{Image token attribution on four MMStar questions.} Each method's top 20\% of patches is outlined, with the share of the measured dependence it captures in the corner (higher is better). The measured dependence is the likelihood drop from blurring each patch in turn. \textsc{VTRACE} leads on the first two questions, ties FlowTracer on the third, and trails AttnLRP on the fourth.}
    \label{fig:image_attribution_casestudy}
\end{figure}

\textbf{Joint image and text attribution.} Figures~\ref{fig:casestudy_mmstar136} to~\ref{fig:casestudy_mmstar475} show four case studies, comparing IFR, a second baseline and \textsc{VTrace} on the same fixed model response, with the boxed final answer (without loss of generality) selected as the target. In each row, the image on the left and the question text outlines the method's top 20\% of the image tokens, and dims the rest. The text on the right shows the question and the full response, with each token shaded by its attribution score: the darker the token, the more it contributed to the answer. 

A faithful attribution should point to the evidence the answer relies on: the image regions the question is about, the key words of the question, and the reasoning steps that carry what the model saw to its answer. Evidently, \textsc{VTrace} does this more consistently than the baselines. Its top patches concentrate on the relevant objects and labels, while the baselines often spread their patches over the background and the image border. In the text, \textsc{VTrace} highlights the key question words and the reasoning steps that use visual evidence, whereas the baselines focus mostly on the answer options and the final sentence. \Cref{fig:casestudy_mmstar475} shows a harder case, where \textsc{VTrace}'s image selection is less focused. See the captions for detailed description.

\begin{figure}
    \centering
    \includegraphics[width=\linewidth]{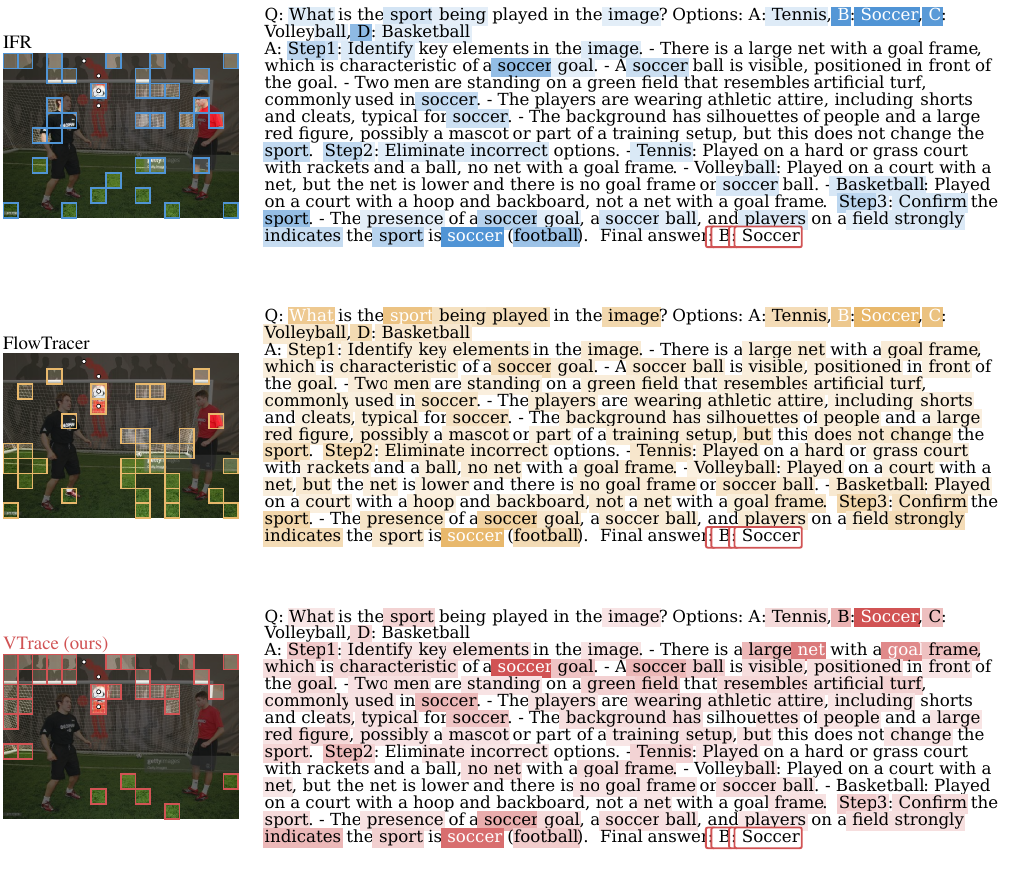}
    \vspace{-0.5cm}
    \caption{\textbf{Recognizing the sport.} The model answers ``Soccer''. \textsc{VTrace} places most of its top patches on the goal net, the object that identifies the sport, while IFR and FlowTracer spread more patches over the players and the grass. In the text, \textsc{VTrace} scores the option ``Soccer'' and the concluding ``soccer'' highest, and also shades the scene the model describes, such as ``large net'', ``goal frame'' and ``soccer ball''; FlowTracer concentrates on the question wording.}
    \label{fig:casestudy_mmstar136}
\end{figure}

\begin{figure}
    \centering
    \includegraphics[width=\linewidth]{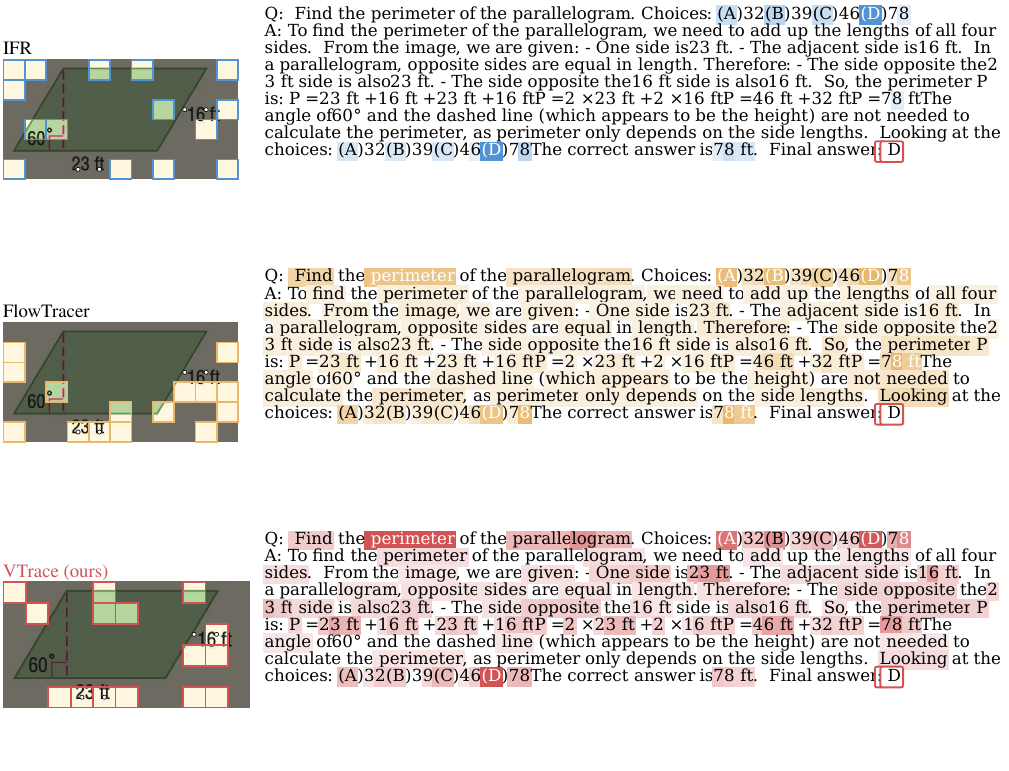}
    \vspace{-0.5cm}
    \caption{\textbf{Reading the side lengths of a parallelogram.} The model adds the side lengths shown in the image, $2\times 23 + 2\times 16 = 78$, and answers (D). \textsc{VTrace}'s top patches cover the ``23 ft'' and ``16 ft'' labels, while IFR's lie mostly along the image border. In the text, IFR highlights little beyond the answer options, whereas \textsc{VTrace} highlights ``perimeter'' in the question and the side lengths and sums in the reasoning that lead to the answer.}
    \label{fig:casestudy_mmstar1158}
\end{figure}

\begin{figure}
    \centering
    \includegraphics[width=1\linewidth]{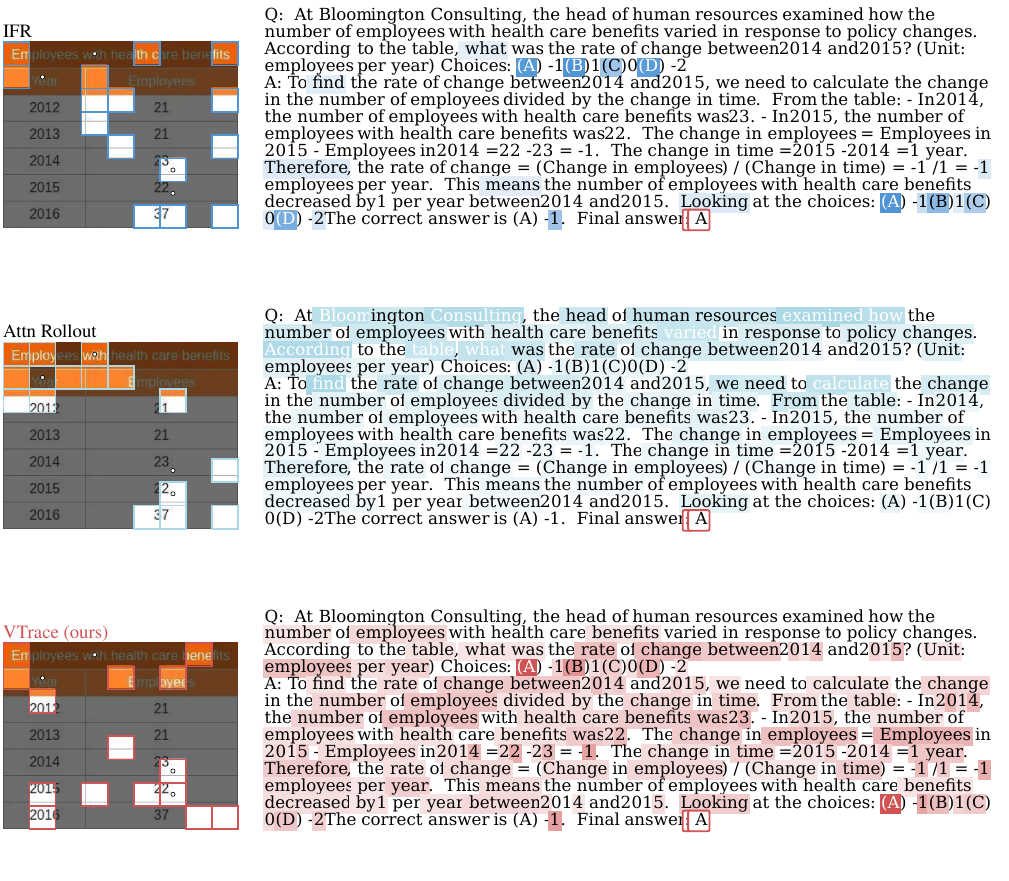}
    \caption{\textbf{Reading values from a table.} The model reads that the count fell from 23 in 2014 to 22 in 2015 and answers (A) $-1$. \textsc{VTrace}'s top patches fall on the year column and the 2014 and 2015 rows, while Attention Rollout concentrates on the table's title bar. In the text, Attention Rollout highlights setup words such as ``Bloomington Consulting'' and ``calculate'', whereas \textsc{VTrace} highlights the years, the change in employees and the result $-1$.}
    \label{fig:casestudy_mmstar892}
\end{figure}

\begin{figure}
    \centering
    \includegraphics[width=1\linewidth]{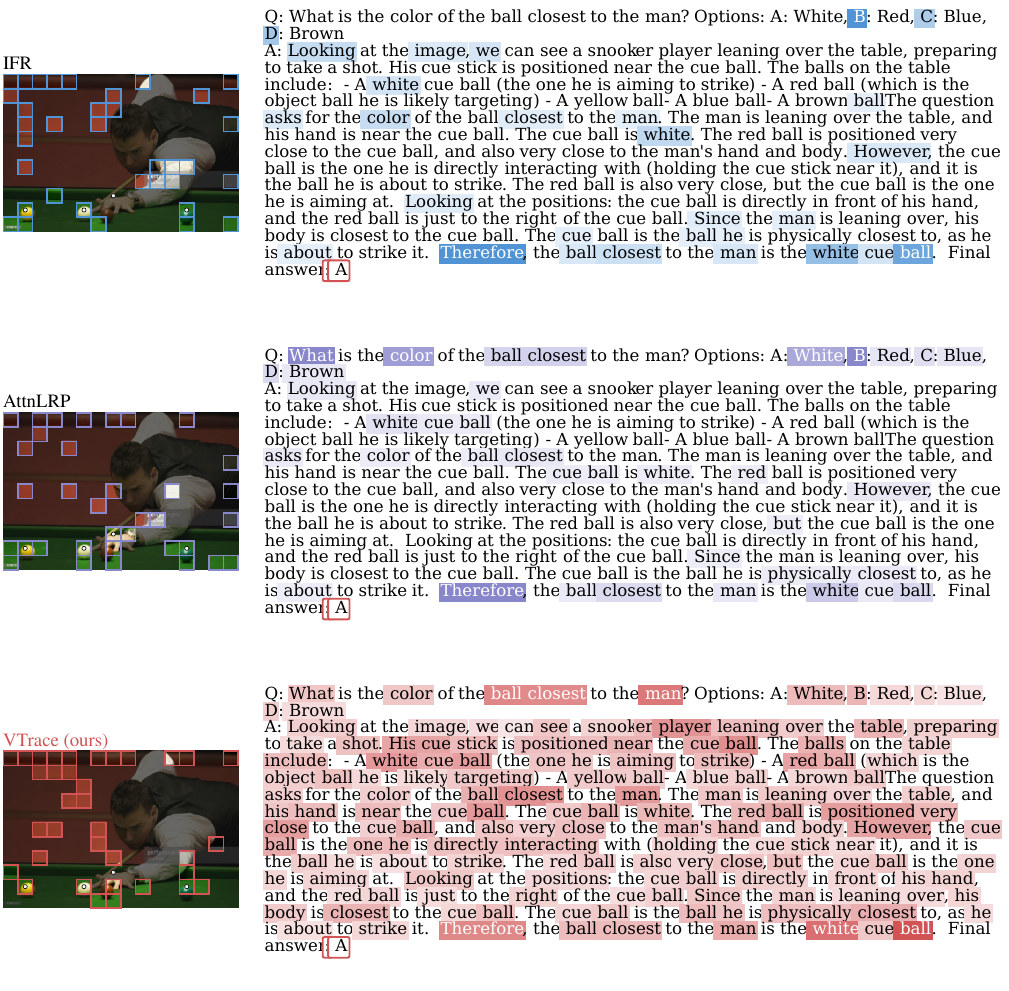}
    \caption{\textbf{A harder case: the ball closest to the man.} The model answers (A) White. In the text, \textsc{VTrace} highlights the key question words ``ball closest'' and ``man'' and the reasoning about the white cue ball. In the image, its top patches cover the balls on the table, including the white cue ball and the red ball at the man's hand, but many also fall on the dark background above the table.}
    \label{fig:casestudy_mmstar475}
\end{figure}

\end{document}